\documentclass{article}
\usepackage{ijcai24}

\usepackage{times}
\usepackage{soul}
\usepackage{url}
\usepackage[hidelinks]{hyperref}
\usepackage[utf8]{inputenc}
\usepackage[small]{caption}
\usepackage{graphicx}
\usepackage{amsmath}
\usepackage{amsthm}
\usepackage{booktabs}
\usepackage{algorithm}
\usepackage{algorithmic}
\usepackage[switch]{lineno}
\usepackage{enumitem}
\usepackage{amsfonts}
\usepackage{multirow}
\usepackage{colortbl}
\usepackage{subfigure}
\usepackage{xcolor}
\usepackage{tikz}
\usetikzlibrary{calc}
\usepackage{bbding}
\usepackage{caption}       

\newtheorem{definition}{\bf Definition}
\newtheorem{proposition}{\bf Proposition}
\newtheorem{assumption}{\bf Assumption}
\newtheorem{thm}{\bf Theorem}
\newtheorem{lem}[thm]{\bf Lemma}

\title{Boosting Efficiency in Task-Agnostic Exploration through Causal Knowledge}

\author{
Yupei Yang$^1$
\and
Biwei Huang$^{2}$\footnote{corresponding author}\and
Shikui Tu$^{1}$\footnotemark[1]\And
Lei Xu$^{1,3}$
\affiliations
$^1$Department of Computer Science and Engineering, Shanghai Jiao Tong University, China\\
$^2$Halicioğlu Data Science Institute (HDSI), University of California San Diego, USA\\
$^3$Guangdong Institute of Intelligence Science and Technology, Zhuhai, China
\emails
\{yupei\_yang, tushikui, leixu\}@sjtu.edu.cn,
bih007@ucsd.edu \\
}

\begin{document}

\maketitle

\begin{abstract}
    The effectiveness of model training heavily relies on the quality of available training resources. However, budget constraints often impose limitations on data collection efforts. To tackle this challenge, we introduce \textit{causal exploration} in this paper, a strategy that leverages the underlying causal knowledge for both data collection and model training. We, in particular, focus on enhancing the sample efficiency and reliability of the world model learning within the domain of task-agnostic reinforcement learning. During the exploration phase, the agent actively selects actions expected to yield causal insights most beneficial for world model training. Concurrently, the causal knowledge is acquired and incrementally refined with the ongoing collection of data. We demonstrate that causal exploration aids in learning accurate world models using fewer data and provide theoretical guarantees for its convergence. Empirical experiments, on both synthetic data and real-world applications, further validate the benefits of causal exploration. The source code is available at \href{https://github.com/CMACH508/CausalExploration}{https://github.com/CMACH508/CausalExploration}.
\end{abstract}

\section{Introduction}
Deep neural networks have been incredibly successful in various domains, such as milestone achievements in Go games and control tasks \cite{silver-Go-2016,tassa-deepmind-2018}. One key factor contributing to such remarkable performance is the availability of high-quality data for model training. However, in many practical applications, it remains data-hungry due to limited data collection efforts imposed by budget constraints \cite{fang-learning-2017,yoo-learning-2019,robine-transformer-2023}. This highlights the essential need to enhance both the sampling and model learning efficiency.

In this paper, we introduce \textit{causal exploration} to tackle this challenge, a novel framework that makes use of the underlying causal knowledge to boost the data collection and model training processes. On one hand, acquiring and understanding causal knowledge unveils the fundamental mechanisms behind the data generation process, thereby reducing the exploration space. In contrast to random data collection, causal exploration allows systematic action planning based on the identified causal structures. On the other hand, causal knowledge reflects the cause-and-effect dependencies among variables. Through the incorporation of causal structural constraints into the model, we can acquire causal dynamics models that eliminate redundant dependencies, as opposed to non-causal dense models, which have shown to provide more accurate estimations \cite{seitzer-causal-2021,huang-action-2022,wang-causal-2022}.

Specifically, we focus on boosting the sample efficiency and reliability of the world model in the realm of task-agnostic reinforcement learning (RL). Different from methods that learn a fixed task from scratch, task-agnostic RL agent first learns a global world model that gathers information about the true environment on the data collected during exploration. Then based on the predictions of the learned model, the agent could make quick adaptations to various downstream tasks in a zero-shot manner given task-specific reward functions. This learning setup exhibits excellent generalization performance but also imposes high requirements on the accuracy of the model \cite{pathak-self-2019}. However, the data collection and world model learning processes are usually expensive due to extensive environment interactions, especially in large state spaces where discovering the optimal policy can be highly challenging \cite{burda-large-2018}.

To address this problem, our causal exploration-based approach revolves around three primary aspects. First, we employ constraint-based methods to discover causal relationships among environment variables. Second, we formulate the dynamics model under causal structural constraints to enhance its reliability. Third, we propose several ways based on causal knowledge to improve the sample efficiency during exploration. In particular, for causal discovery, we present an efficient online method that selectively eliminates noisy samples and strategically gathers informative data points in an incremental way. Moreover, we learn to actively explore towards novel states that are expected to contribute most to model training. The learning process of the exploration policy is driven by intrinsic rewards, which measure both the agent's level of surprise at the outcome and the quality of the training caused by the selected data. During exploration, causal knowledge and the world model are continuously refined with the ongoing collection of data. Our key contributions are summarized below.
\begin{itemize}
  \item In order to enhance the sample efficiency and reliability of model training with causal knowledge, we introduce a novel concept: causal exploration, and focus particularly on the domain of task-agnostic reinforcement learning.
  \item To efficiently learn and use causal structural constraints, we develop an online method for causal discovery and formulate the world model with explicit structural embeddings. During exploration, we train the dynamics model under a novel weight-sharing-decomposition schema that can avoid additional computational burden.
  \item Theoretically, we show that, given strong convexity and smoothness assumptions, our approach attains a superior convergence rate compared to non-causal methods. Empirical experiments further demonstrate the robustness of our online causal discovery method and validate the effectiveness of causal exploration across a range of demanding reinforcement learning environments.
\end{itemize}

\section{Preliminary}
We consider task-agnostic RL within the framework of a Markov decision process characterized by state space $\mathcal{S}$ and action space $\mathcal{A}$. In addition, to integrate causal information, we make the following assumptions throughout this paper.
\begin{assumption}
    (Causal Factorization). 
    Both the state space and action space can be factorized. That is, $\mathcal{S} = \mathcal{S}_1 \times \ldots \times \mathcal{S}_n \in \mathbb{R}^n$ and $\mathcal{A} = \mathcal{A}_1 \times \ldots \times \mathcal{A}_c \in \mathbb{R}^c$.
\end{assumption}
\begin{assumption}
    (Causal Sufficiency). 
    The state variables $\boldsymbol{s}_t$ are fully observable without any hidden confounders.
\end{assumption}
\begin{assumption}
    (Faithfulness Condition). 
    For a causal graph $\mathcal{G}$ and the associated probability distribution $P$, every true conditional independence relation in $P$ is entailed by the Causal Markov Condition applied to $\mathcal{G}$.
\end{assumption}
Given these commonly made assumptions for causal discovery methods, we next define \textit{transition causality} over the transition variables from $\{\boldsymbol{s}_{t-1}, \boldsymbol{a}_{t-1}\}$ to $\boldsymbol{s}_t$.
\begin{definition}\label{def:causality}
    (Transition Causality). 
    Under the Markov condition, the causal structures are over the state-action variables $\mathcal{U} = \{\mathcal{S}_{i,t-1}\}_{i=1}^n \cup \{\mathcal{A}_{j,t-1}\}_{j=1}^c$ and $\mathcal{V} = \{\mathcal{S}_{i,t}\}_{i=1}^n$, which can be represented by a directed acyclic graph $\mathcal{G} = (\{ \mathcal{U}, \mathcal{V}\}, \mathcal{E})$ and its adjacency matrix $D$. Here, $\mathcal{E}$ denotes the edge set and $D \in {\{0, 1\}}^{|\mathcal{U}| \times |\mathcal{V}|}$.
\end{definition}

Note that all edges are from $\mathcal{U}$ to $\mathcal{V}$. If $s_{i,t-1} \in \mathcal{U}$ has a causal edge to $s_{j,t} \in \mathcal{V}$, then we call $s_{i, t-1}$ a parent of $s_{j,t}$ and have $D(i,j)=1$. Take the example in Figure \ref{fig:causality}: we have $D(1, 1) = 1$ because $s_{1,t-1}$ is a parent of $s_{1,t}$, while $D(2,1)=0$ because $s_{2,t-1}$ does not have a causal edge to $s_{1,t}$. Here, we assume that the structural constraints are invariant over time $t$. The causal identification theory under these appropriate definitions and assumptions has been given in existing work.
\begin{proposition}
    Under the aforementioned assumptions, the causal adjacency matrix $D$ is identifiable (see \cite{huang-adarl-2021,wang-causal-2022,ding-generalizing-2022}).
\end{proposition}

\section{Related Work}
\paragraph{Task-agnostic RL.}
Over recent decades, task-agnostic exploration strategies have been an active research area to attain generalization \cite{aubret-survey-2019}. Existing methods focus on designing appropriate forms of intrinsic rewards, which can be broadly categorized into three types: (1) the number of times a state has been visited, which helps to guide the agent towards unexplored regions \cite{bellemare-unifying-2016,machado-count-2020}; (2) curiosity about the environment dynamics, which is usually formalized as prediction errors of future states \cite{pathak-curiosity-2017,kim-emi-2018}; and (3) information gain, which aims to improve the agent's knowledge about the environment by maximizing the mutual information \cite{duan-benchmarking-2016,shyam-model-2019}. However, existing works usually pay little attention to the support that causal structures can offer in improving exploration efficiency.

\paragraph{RL with causal discovery.}
The intersection of causal discovery and RL has become a popular trend in recent years. \cite{zhu-causalrl-2019} uses RL to search for the causal graph with the best score. \cite{peng-causality-2022} proposes to learn a hierarchical causal structure for subgoal-based policy learning. \cite{ding-generalizing-2022} and \cite{mutti-2023-provably} focus on providing tractable formulations of systematic generalization in RL tasks by employing a causal viewpoint. \cite{yu-explainable-2023} learns a causal world model to generate explanations for the decision-making process. Theoretical evidence about the advantages of using a causal world-model in offline RL is given in \cite{zhu-2022-offline}. However, most of these works focus on either extracting underlying causal graphs from given data in a particular environment or using random exploration strategies for data collection, instead of directly utilizing causal structure as guidance to improve exploration efficiency in task-agnostic RL \cite{kosoy-learning-2022}. Experimental comparisons with some existing works that are close to us have been conducted in Section \ref{sec:exp}.

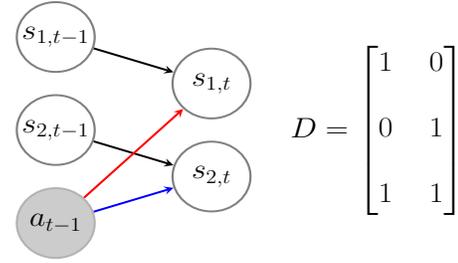
\begin{figure}[tbp]
    \centering
    \resizebox{.72\linewidth}{.4\linewidth}{
    \begin{tikzpicture}[mycircle/.style={circle,draw=black!50,fill=white!20,very thick, minimum size=15mm, inner sep=0pt, font=\fontsize{18}{18}\selectfont},myfirstnode/.style={circle,draw=black!30,fill=black!20,very thick, minimum size=15mm, inner sep=0pt, font=\fontsize{18}{18}\selectfont}]
        \node at (0.2, 0) {\huge $D = \begin{bmatrix} 1 && 0 \\ && \\0 && 1 \\ && \\ 1 && 1 \end{bmatrix}$}; 
        \node (n21) at (-3,1) [mycircle] { $ s_{1, t}$};
        \node (n22) at (-3,-1) [mycircle] {$ s_{2, t}$};
        \node (s21) at ($(n22)+(-3,3)$) [mycircle] {$ s_{1, t-1}$}edge[->, bend angle=0, draw=black, >=stealth, line width=1.2pt](n21) ;
        \node (s22) at ($(n22)+(-3,1)$) [mycircle] {$ s_{2, t-1}$}edge[->, bend angle=0, draw=black, >=stealth, line width=1.2pt](n22);
        \node (s24) at ($(n22)+(-3,-1)$) [myfirstnode] {$ a_{t-1}$}edge[->, bend angle=0, draw=red, >=stealth, line width=1.2pt](n21) edge[->, bend angle=0, draw=blue, >=stealth, line width=1.2pt](n22);
    \end{tikzpicture}
    }
    \caption{An illustration of causal relationships from $\mathcal{U}=\{s_{1,t-1}, s_{2,t-1}, a_{t-1}\}$ to $\mathcal{V}=\{s_{1,t}, s_{2,t}\}$ in the RL system.}
    \label{fig:causality}
\end{figure}

\section{Discovering and Utilizing Causality for Learning World Models}
After establishing proper assumptions and definitions, we proceed to introduce the methodology part for causal exploration. In this section, we initially assume that the causal structures are known and show how to explicitly incorporate causal knowledge into the world model and utilize it for model training. Next, we give an estimation procedure for causal structures.

\subsection{Causal Constraints for Forward Model}
During task-agnostic exploration, the agent learns a world model $f_{\boldsymbol{w}^c}$ with parameter $\boldsymbol{w}^c$, serving as an abstraction of the ground truth transitions in the environment. In other words, the world model $f_{\boldsymbol{w}^c}$ is designed to enable agents to predict future state $\hat{\boldsymbol{s}}_{t}$ based on current state $\boldsymbol{s}_{t-1}$ and action $\boldsymbol{a}_{t-1}$, represented by $\hat{\boldsymbol{s}}_t = f_{\boldsymbol{w}^c}(\boldsymbol{s}_{t-1}, \boldsymbol{a}_{t-1}, \boldsymbol{e}_t)$, where $\boldsymbol{e}_t$ is the corresponding random noise. However, based on the understanding that causal structures within environmental variables are typically sparse rather than dense, as suggested by \cite{huang-action-2022}, such a framework could contain unnecessary dependencies. For instance, in the context of Figure \ref{fig:causality}, the variable $s_{2,t-1}$ does not causally affect $s_{1,t}$, and is thus identified as a non-parent node. Consequently, we only need to take a subset of $(\boldsymbol{s}_{t-1}, \boldsymbol{a}_{t-1})$ as inputs of the model.

\begin{figure}[tbp]
    \centering
    \includegraphics[width=\linewidth]{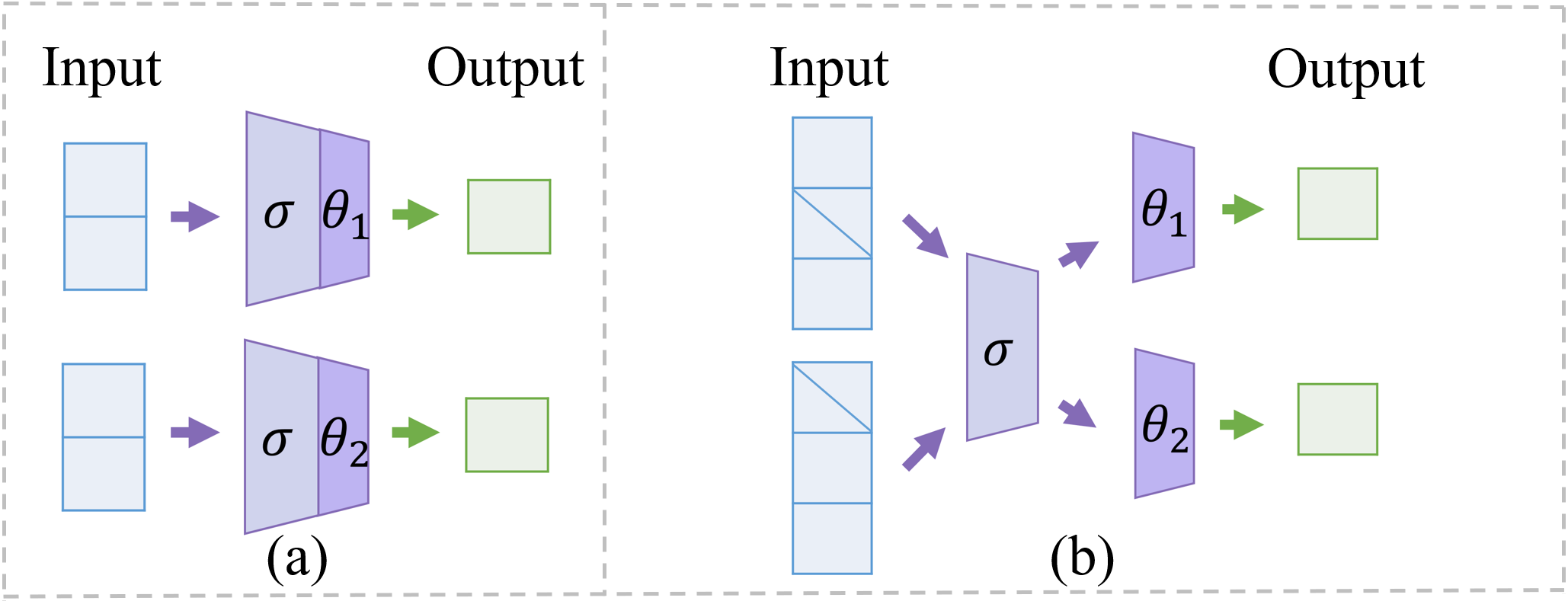}
    \caption{An illustration of model architectures under (a) naive causal factorization; (b) our sharing-decomposition schema.}
    \label{fig:sharing}
\end{figure}

To reflect these constraints, we explicitly consider the causal structures over state and action variables to model the one-step transition dynamics, formulated as:
\begin{equation}
    \begin{aligned}
       \hat{\boldsymbol{s}}_t = \prod_{i=1}^n f^i_{\boldsymbol{w}^c} (D_i \odot (\boldsymbol{s}_{t-1}, \boldsymbol{a}_{t-1}), \boldsymbol{e}_{i,t}),
    \end{aligned}
\end{equation}
where $\odot$ denotes element-wise product, $D_i$ and $\boldsymbol{e}_{i,t}$ are the $i$-th column of causal matrix $D$ and noise term $\boldsymbol{e}_t$, respectively. However, a naive implementation requires training $n$ world models since each factored dimension has its unique parents. It is likely to result in an explosive growth in computational complexity as the state dimension and network size increase. 

We propose a sharing-decomposition schema to address such a problem. It is unnecessary for all of these $n$ networks to be totally different. Instead, each of these models could share the first several layers as a common embedding. Then following the sharing module, we design predictive networks for each dimension. Suppose $\boldsymbol{w}_i^c$ is the network parameter for the $i$-th dimension, it is a combination of the shared parameter $\boldsymbol{\sigma}$ and decomposed parameter $\boldsymbol{\theta}_i$, written as
\begin{equation}
    \boldsymbol{w}_i^c = \boldsymbol{\sigma} \cup \boldsymbol{\theta}_i.
\end{equation}
During the training time, each model focuses on a different aspect of the state in the decomposition part $\boldsymbol{\theta}_i$ but shares a common knowledge $\boldsymbol{\sigma}$. The number of shared layers is a hyperparameter that allows for a trade-off between the sharing and decomposition parts. By training forward models under this schema, our approach can both utilize causal information of the ground environment dynamics to generate accurate predictions and achieve a significant reduction in model parameters and computation time compared with naive decomposition. Figure \ref{fig:sharing} illustrates our sharing-decomposition schema during causal exploration. Moreover, we use the prediction error as the loss function for the optimization of $\boldsymbol{w}^c$:
\begin{equation}\label{equ:pred_loss}
  \begin{aligned}
    L_f(\boldsymbol{w}^c) = \frac12 \| \hat{\boldsymbol{s}}_{t} - \boldsymbol{s}_{t} \|^2_2.
  \end{aligned}
\end{equation}

\subsection{Efficient Online Causal Relationship Discovery}\label{sec:causal}
In this section, we show how to identify the causal adjacency matrix $D$. According to Definition \ref{def:causality}, this can be transformed into determining whether there exists an edge between each pair of nodes in the causal graph $\mathcal{G}$. To achieve this goal, we start from a complete graph and then iteratively remove unnecessary edges using Conditional Independence Tests (CIT). Given that the edges follow the temporal order without instantaneous connection, we extend the PC algorithm \cite{spirtes-causation-2000} to handle time-lagged causal relationships based on Kernel-based Conditional Independence (KCI) test \cite{zhang-kernel-2012} to identify the causal adjacency matrix $D$.

\begin{figure}
    \centering
    \includegraphics[width=\linewidth]{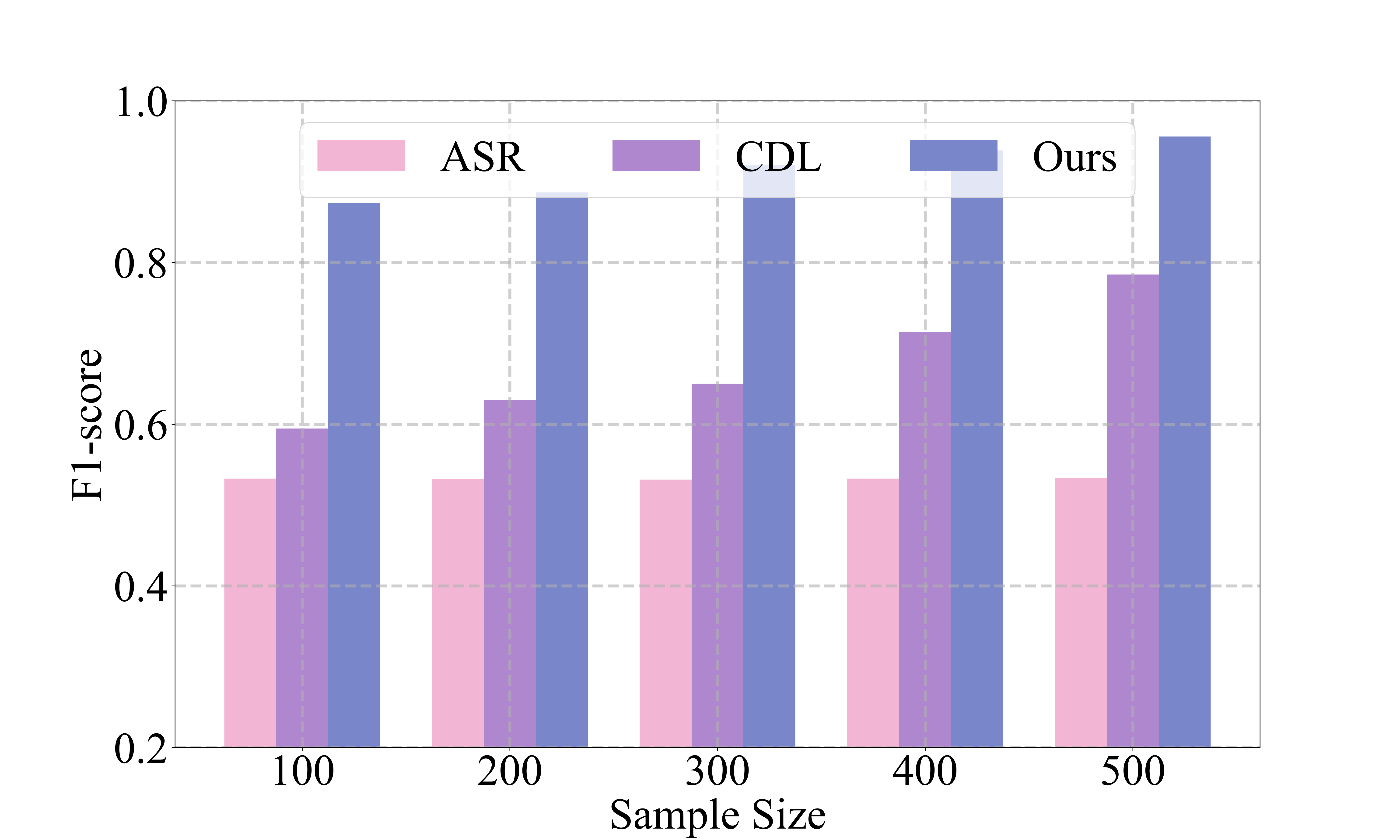}
    \caption{Average F1-score of the identified causal graph compared to the ground truth graph across various causal discovery methods. More details are given in Section \ref{sec:exp}.}
    \label{fig:graph_score}
\end{figure}

\begin{figure*}[htbp]
  \centering
  \includegraphics[width=\textwidth]{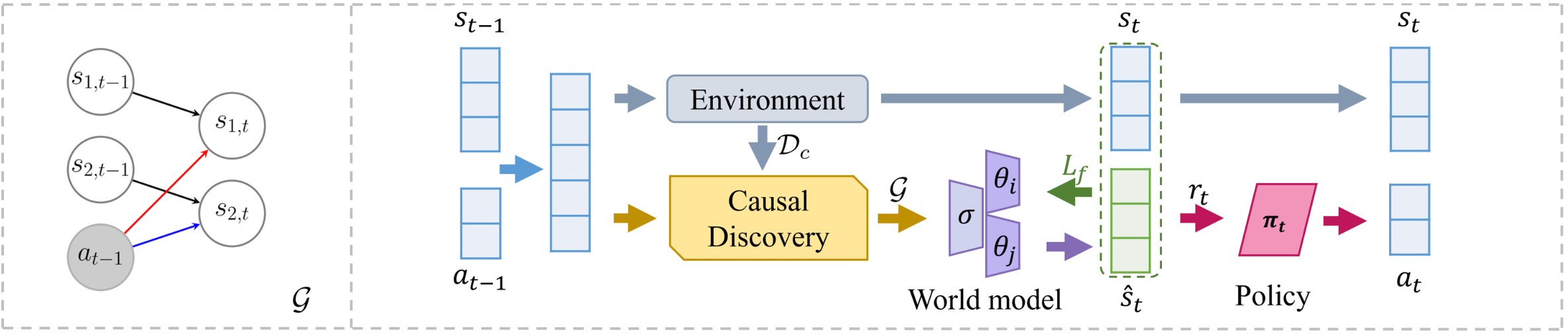}
  \caption{An overview of Causal Exploration framework. Throughout the process, the agent, guided by policy $\pi_t$, engages in exploration to gather data that are most beneficial for model training. Meanwhile, causal knowledge and the world model are continuously refined with the ongoing data collection.}
  \label{fig:causalnet}
\end{figure*}

While causal discovery algorithms typically necessitate the collection of substantial causal information through data, it's important to note that accumulating more samples does not always confer an advantage: as the sample size increases, the time cost of causal algorithms also rises. Figure \ref{fig:pc-time} in the appendix illustrates that the execution time of the PC algorithm based on KCI tests experiences exponential growth as sample size increases. Therefore, prioritizing the enhancement of data quality over quantity becomes paramount. In order to reduce the cost of the identification process, we design an efficient online causal relationship discovery method: instead of using all of the coming data for causal identification, we selectively collect representative data points during exploration in an incremental way. Specifically, we use the minibatch similarity and sample diversity criteria introduced in \cite{yoon-online-2021} as our selection strategies, which are defined as
\begin{equation}
\begin{aligned}
    \text{Similarity} &= \frac{\nabla f_{\boldsymbol{w}^c}({b^i_t}) \overline{\nabla}f_{\boldsymbol{w}^c}({\mathcal{B}_t})^\top}{\| \nabla f_{\boldsymbol{w}^c}({b^i_t}) \| \cdot  \|\overline{\nabla}f_{\boldsymbol{w}^c}({\mathcal{B}_t}) \|}, \\
    \text{Diversity} &= \frac{-1}{t - 1} \sum_{p \neq i}^{t -1} \frac{\nabla f_{\boldsymbol{w}^c}({b^i_t}) {\nabla}f_{\boldsymbol{w}^c}({b^p_t})^\top}{\| \nabla f_{\boldsymbol{w}^c}({b^i_t}) \| \cdot \|{\nabla}f_{\boldsymbol{w}^c}({b^p_t}) \|}.
\end{aligned}
\end{equation}
Here ${b^i_t}$ is the $i$-th data of the whole arrival batch ${\mathcal{B}_t}$. $\nabla f_{\boldsymbol{w}^c}({b^i_t})$ and $\overline{\nabla}f_{\boldsymbol{w}^c}({\mathcal{B}_t})$ are the gradient and average gradient of the sample and batch, respectively. A combination of these two criteria is used to select the top-$\kappa$ data among $\mathcal{B}_t$ for causal discovery during exploration. Experiments in Section \ref{sec:exp} show that this online method significantly accelerates causal discovery without sacrificing overall accuracy. 

It is noteworthy that our approach allows for flexibility in the choice of causal discovery methods. However,
as has been discussed in \cite{ding-generalizing-2022}, constraint-based approaches usually exhibit better robustness compared with score-based methods under our specific configuration. Furthermore, we have also conducted comparative experiments among different causal discovery methods to demonstrate the correctness of the causal structure identified by the time-lagged PC algorithm in Section \ref{sec:exp}.

\section{Boosting Efficiency through Causal Exploration}\label{sec:policy}
We now return to the fundamental question: how to enhance the data collection efficiency during causal exploration, thereby improving the performance of both causal discovery and model learning. To attain this goal, a commonly applied concept from active learning is the selection of samples that make the largest contributions to the model's training loss. These samples are typically considered as a subset that the model is least familiar with. Hence, the prediction loss is used here as the intrinsic reward to guide exploration with a scaling weight $\eta$:
\begin{equation}\label{equ:intrinsic}
  \begin{aligned}
    r^i_{t-1} = \frac{\eta}{2} \| \hat{\boldsymbol{s}}_{t} - \boldsymbol{s}_{t} \|^2_2.
  \end{aligned}
\end{equation}
This prediction loss can also be viewed as a validation of the agent's causal beliefs. The larger the prediction error, the more surprised the agent is by the actual outcome, implying a greater deviation from the estimated values based on the causal structure and the world model. The faster the error rate drops, the more learning progress signals we acquire.

However, not all the novel states have a positive impact on the model. On the contrary, some noisy data may contribute significantly to prediction errors but can lead the model to an awful direction, which necessitates the agent to also pay attention to the inherent quality of the data during exploration. To reflect this, we introduce active reward \cite{fang-learning-2017} that measures the data quality as another intrinsic motivation. Once a new sample is collected at time step $t$, active reward is then calculated as the change of the model's prediction ability before and after training. We use the prediction accuracy on a test set $\mathcal{D}_h$ generated from episodes unseen before training to reflect the world model's performance and formulate active reward as
\begin{equation}
  \begin{aligned}
    r^a_{t-1} = \overline{L_f(\boldsymbol{w}^c_{t-2})} - \overline{L_f(\boldsymbol{w}^c_{t-1})},
  \end{aligned}
\end{equation}
where $\overline{L_f(\cdot)}$ is the mean prediction error on $\mathcal{D}_h$ and $\boldsymbol{w}^c_t$ denotes parameters of the trained world model at time $t$. The value of active reward reflects beneficial or detrimental training caused by the selected data. If the reward is always positive, it indicates that the agent has been selecting beneficial samples for training the world model. We combine prediction loss and active reward with a regularization weight $\beta$:
\begin{equation}\label{equ:active}
    \begin{aligned}
        r_{t-1} = r^i_{t-1} + \beta r^a_{t-1}.
    \end{aligned}
\end{equation}
During causal exploration, the agent keeps searching for causal informative data by maximizing the expected rewards, which is
\begin{equation}\label{eq:policy}
    \boldsymbol{a}^*_t = \underset{\boldsymbol{a} \in \mathcal{A}}{\arg\max} ~~ \mathbb{E}_{\tau \sim \pi} \left[\sum_t \gamma^t r_t \right],
\end{equation}
where $\tau$ represents the trajectory generated by the exploration policy $\pi$ and $\gamma$ is the discount factor. Meanwhile, the world model minimizes the prediction loss. Since both of them contain the prediction error in equation (\ref{equ:pred_loss}) and (\ref{equ:intrinsic}), we can draw a conclusion that the learning of world models and causal exploration facilitate each other. Figure \ref{fig:causalnet} shows an overview of causal exploration, details are shown in Algorithm \ref{alg:framework}.
\begin{algorithm}[htb]
  \caption{Task-agnostic Causal Exploration}
  \label{alg:framework}
  \begin{algorithmic}[1]
  \STATE \textbf{Initialize}: ~Forward world Model $f$ with parameter $\boldsymbol{w}^c$ \\ \qquad \qquad ~ Causal discovery period $N$, Dataset $\mathcal{B}_t$\\ \qquad \qquad ~ Exploration policy $\pi$ with memory buffer $\mathcal{M}$
  \STATE \textbf{Output}: Forward world model $f$
  \FOR{Episode = $1,2,\ldots$}
  \STATE Collect test set $\mathcal{D}_h$ for the calculation of active reward
  \STATE Set current time $t=1$
  \WHILE{$t<T$}
  \STATE Choose action $\boldsymbol{a}_{t-1}$ and predict next state $\hat{\boldsymbol{s}}_{t}$
  \STATE Obtain $\boldsymbol{s}_{t}$ from environment
  \STATE Calculate intrinsic reward $r_{t-1}$
  \STATE Store $(\boldsymbol{s}_{t-1}, \boldsymbol{a}_{t-1}, \boldsymbol{s}_{t})$ into $\mathcal{B}_t$
  \STATE Store $(\boldsymbol{s}_{t-1}, \boldsymbol{a}_{t-1}, r_{t-1}, \boldsymbol{s}_{t})$ into $\mathcal{M}$
  \IF{$t=0 \pmod N$}
  \STATE Do online causal discovery on top-$\kappa$ of $\mathcal{B}_t$ and get causal graph $\mathcal{G}$
  \ENDIF
  \STATE Train world model under $\mathcal{G}$ on $\mathcal{B}_t$ and update $\boldsymbol{w}^c$
  \STATE Train exploration policy on $\mathcal{M}$ and update $\pi$
  \ENDWHILE
  \ENDFOR
  \RETURN Latest forward world model $f$
  \end{algorithmic}
\end{algorithm}

\subsection{Theoretical Analysis on Causal Exploration}
In this subsection, we first present a mathematical criterion for evaluating the impact of causal exploration on sampling efficiency, and then we theoretically demonstrate the benefits of causal exploration in learning world models, especially when the causal graph is sparse. 

\paragraph{Causal efficiency ratio.}
Let $k$ denote the optimization step in world model training, $f$ represent the world model, and $L_f(\boldsymbol{w})$ be the corresponding loss function. The model parameters incorporating causal knowledge after the $k$-th optimization are denoted as $\boldsymbol{w}^c(k)$, while $\boldsymbol{w}(k)$ represents the parameters without causal information \footnote{For each exploration time $t$, the model undergoes several optimization steps, influenced by the data volume and batch size.}. We use the convergence bounds ratio, denoted by $\xi$, calculated as the relative improvement in the model's loss compared to the optimal value $L_f^\star$, to assess the enhanced exploration efficiency resulting from the incorporation of causal knowledge:
\begin{equation}
    \xi = \frac{L_f(\boldsymbol{w}^c(k)) - L_f^\star}{L_f(\boldsymbol{w}(k)) - L_f^\star}.
\end{equation}
Note that this ratio measures how much closer the model with causal knowledge gets to the optimal value compared to those without causal information. If $\xi < 1$ consistently holds, it indicates an improvement in efficiency attributable to causal exploration. Next, we provide a theoretical analysis of the performance of causal exploration in linear cases, and empirical experiments are further conducted to show that causal exploration is still efficient in learning deep neural networks. 

\paragraph{Theoretical guarantees.}
We consider learning a linear forward model with gradient descent, formulated as
\begin{equation}
  \begin{aligned}
    \min_{\boldsymbol{w}} \quad L_f(\boldsymbol{w}) = \|\boldsymbol{w}^\top \cdot (\boldsymbol{s}_{t-1}, \boldsymbol{a}_{t-1}) - \boldsymbol{s}_t \|^2_2,
  \end{aligned}
\end{equation}
where $\boldsymbol{s}_t={\boldsymbol{w}^\star}^\top \cdot (\boldsymbol{s}_{t-1}, \boldsymbol{a}_{t-1})$ is the ground truth value. Moreover, we make the following assumption:
\begin{assumption}\label{ass:convex}
  $L_f$ is strong convex and smooth such that $\exists ~m>0,~M>0$, for any $\boldsymbol{w} \in$ $\rm dom$ $L_f$, we have
  \begin{equation*}
      MI \succeq \nabla^2 L_f(\boldsymbol{w}) \succeq mI.
  \end{equation*}
\end{assumption}
The following theorem shows reduced error bound with causal exploration.
\begin{thm}\label{thm}
  Suppose Assumption \ref{ass:convex} holds, and suppose the density of causal matrix $D$ is $\delta$ and the model is initialized with $\boldsymbol{w}_0$. Then for every optimization step $k$, we have
  \begin{equation}\label{equ:ratio}
    \begin{aligned}
      L_f(\boldsymbol{w}^c(k)) - L_f^\star &\le \delta^k \left( L_f(\boldsymbol{w}(k)) - L_f^\star\right) \\
      &\le \frac M2 \left[ \delta  \left( 1- \frac m M \right)\right]^k \| \boldsymbol{w}_{0} - \boldsymbol{w}^\star \|^2_2.
    \end{aligned}
  \end{equation}
\end{thm}
\paragraph{Remark.}
The first inequality in Equation (\ref{equ:ratio}) establishes an upper bound for $\xi$ at $ \delta^k $. Given that $0 \le \delta^k \le 1$, this confirms the effectiveness of causal exploration in enhancing efficiency. The subsequent inequality establishes that the training error of causal exploration gradually converges to the optimal value, exhibiting an \textit{exponential} convergence rate characterized by the decay of the factor $\left[ \delta  \left( 1- \frac m M \right)\right]^k$. Moreover, this theorem implies that the advantages of causal exploration are relevant to the sparseness of causal structure. The sparser the causal structure, the faster our method learns. When the causal matrix $D$ is a complete matrix ($\delta = 1$), causal exploration degenerates into non-causal prediction-based exploration. The proof for Theorem \ref{thm} is provided in Appendix \ref{ape:proof}.

\section{Experiments}\label{sec:exp}

\begin{table*}[tbp]
	\centering
        \resizebox{\linewidth}{!}{
	\begin{tabular}{lcccc}
		\toprule
		Approach &  Causal Knowledge & Explicit Structural Embedding & Online Causal discovery & Non-random Data Collection\\
		\midrule
		Curiosity \cite{pathak-curiosity-2017} &  \XSolidBrush & \XSolidBrush  & \XSolidBrush & \Checkmark  \\
		Plan2Explore \cite{Pathak-plan2explore-2019} & \XSolidBrush & \XSolidBrush  & \XSolidBrush & \Checkmark  \\
		CID \cite{seitzer-causal-2021} & \Checkmark & \XSolidBrush  & \XSolidBrush & \XSolidBrush  \\
            ASR \cite{huang-action-2022} & \Checkmark & \Checkmark  & \XSolidBrush & \XSolidBrush  \\
            CDL \cite{wang-causal-2022} & \Checkmark & \XSolidBrush  & \XSolidBrush & \Checkmark  \\
		\textbf{Causal Exploration (ours)} & \Checkmark & \Checkmark  & \Checkmark & \Checkmark   \\
		\bottomrule
	\end{tabular}}
        \caption{Attributes of existing baselines and our causal exploration. \Checkmark denotes that a method has an attribute, whereas \XSolidBrush denotes the opposite. "Causal Knowledge" here indicates whether causal information is utilized. Methods with "Explicit Structural Embedding" formulate dynamics models embedded the with causal matrix $D$. "Online Causal Discovery" and "Non-random Data Collection" measure the contribution of the method to boosting sampling and model learning efficiency from different perspectives, corresponding to Sections \ref{sec:causal} and \ref{sec:policy}, respectively.}
	\label{tab:attributes}
\end{table*}

To evaluate the effectiveness of causal exploration in complex scenarios, we conduct a series of experiments on both synthetic datasets and real-world applications, including the traffic light control task and MuJoCo control suites \cite{todorov-mujoco-2012}. Meanwhile, we compare our proposed causal exploration with the following baseline methods:
\begin{enumerate}
    \item \textbf{Curiosity} \cite{burda-large-2018} and \textbf{Plan2Explore} \cite{Pathak-plan2explore-2019}: non-causal exploration methods that learn dense models with different forms of intrinsic rewards.
    \item \textbf{CID} \cite{seitzer-causal-2021}: focusing on detecting causal influences between actions $\boldsymbol{a}_{t-1}$ and future states $\boldsymbol{s}_t$ with random data collection. Causal exploration extends to discover dependencies among state transitions as well.
    \item \textbf{ASR} \cite{huang-action-2022}: a differential approach that learns the causal matrix $D$ as free parameters together with the world model $f$. Instead, the identification of $D$ is completed by causal discovery methods in our approach. 
    \item \textbf{CDL} \cite{wang-causal-2022}: learning causal dynamics models to approximate conditional mutual information rather than directly implementing CIT on the observed data.
\end{enumerate}
Table \ref{tab:attributes} summarizes the attributes of these baselines and our proposed causal exploration from four perspectives. Moreover, following \cite{ke-systematic-2021}, we use the following metrics to evaluate the performance of the learned models:
\begin{enumerate}
    \item \textbf{Prediction Ability}: The error in predicting future states is a direct reflection of the  model learning quality. In other words, lower prediction errors achieved with the same number of samples during exploration imply a higher level of accuracy, signifying increased sampling efficiency.
    \item \textbf{Structural Difference}: Evaluation of causal discovery methods is based on the difference between the identified causal matrix and the ground truth matrix.
    \item \textbf{Downstream Performance}: For task-agnostic exploration, performance on downstream tasks also serves as a good metric to measure the model’s understanding of the underlying causal structures of the environment.
\end{enumerate}

\subsection{Synthetic Datasets}\label{sec:synthetic}
\paragraph{Environments.}
We build our simulated environment following the state space model with controls. When the agent takes an action $\boldsymbol{a}_t$ based on the current state, the environment provides feedback $\boldsymbol{s}_{t+1}$ at the next time. We denote the generative environment as
\begin{equation}
  \begin{aligned}
    \boldsymbol{s}_1 \sim \mathcal{N}(\boldsymbol{0}, I), \quad
    \boldsymbol{s}_t \sim \mathcal{N} (h(\boldsymbol{s}_{t-1},\boldsymbol{a}_{t-1}), \Sigma),
  \end{aligned}
\end{equation}
where $\Sigma$ is the covariance matrix and $h$ is the mean value as the ground truth transition function implemented by deep neural networks under causal graph $\mathcal{G}$. Specifically, the linear condition consists of a single-layer network, and the nonlinear function is three-layer MLPs with sigmoid activation.

Given that the sparsity of the causal graph is an important factor affecting the performance of our method, we choose to demonstrate the superiority of our approach on relatively low-density causal structures. That is, $\mathcal{G}$ is generated by randomly connecting edges with a probability of $p$. Such sparse causal structures allow us to evaluate the ability of our method to accurately learn world models with limited data, which is a common scenario in many real-world applications. Below, we provide discussions on the benefits of causal exploration based on the empirical results.
\begin{figure}[tbp]
    \centering
    \includegraphics[width=\linewidth]{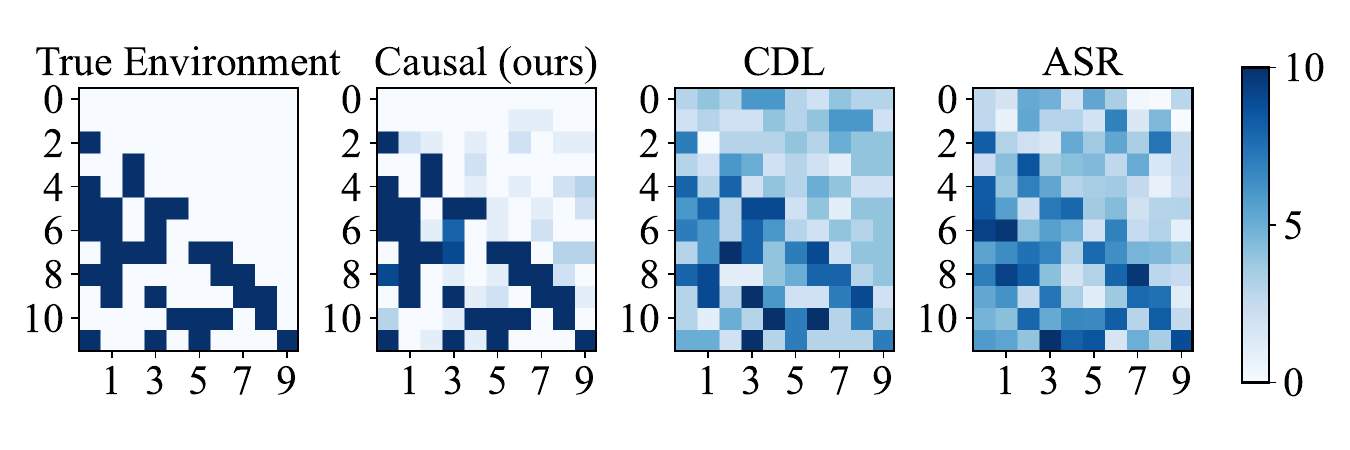}
    \caption{A comparison of the discovered causal matrix with different methods when $|\mathcal{U}| = 12$ and $|\mathcal{V}|=10$.}
    \label{fig:com_causal}
\end{figure}
\paragraph{Correctness of the identified causal structure.} We compare our proposed method with different causal discovery methods. To be specific, CID and CDL apply score-based causal discovery, while ASR combines causal structure with neural networks for differentiable training. However, CID only considers causal action influence and lacks an explicit causal structure, so we do not use it for comparison. Figure \ref{fig:com_causal} illustrates an example of the causal matrix discovered by different methods, where the depth of the color indicates the number of times each edge is detected as a causal connection. Figure \ref{fig:graph_score} further shows the averaged F1-score of these methods. The results reflect that the causal matrix learned through causal exploration obtains a smaller bias from the ground-truth matrix than others, an explanation for better performances.

\paragraph{Reliability of the learned world models.}
To show the advantage of causal exploration for improving the quality of model learning, we compare it with all of the baselines. Figure \ref{fig:synthetic_result} are prediction errors of the world models during exploration under various state and action spaces. For each of these settings, we conduct $10$ experiments and take the average value to reduce the impact of randomness. In all of these environments, causal exploration achieves lower prediction errors with fewer data. Compared to Curiosity and Plan2Explore, these advantages stem from the introduction of causality. In addition, the larger the dimension and the sparser the causal structure, the more prominent the advantages of our proposed method over non-causal method will be. The improvements over ASR and CID can be attributed to the effective design of exploration strategies. Furthermore, the experimental results compared with CDL underscore the importance of the stability of causal discovery methods and the necessity of explicit causal embeddings. 
\begin{figure}[tbp]
  \centering
  \includegraphics[width=\linewidth]{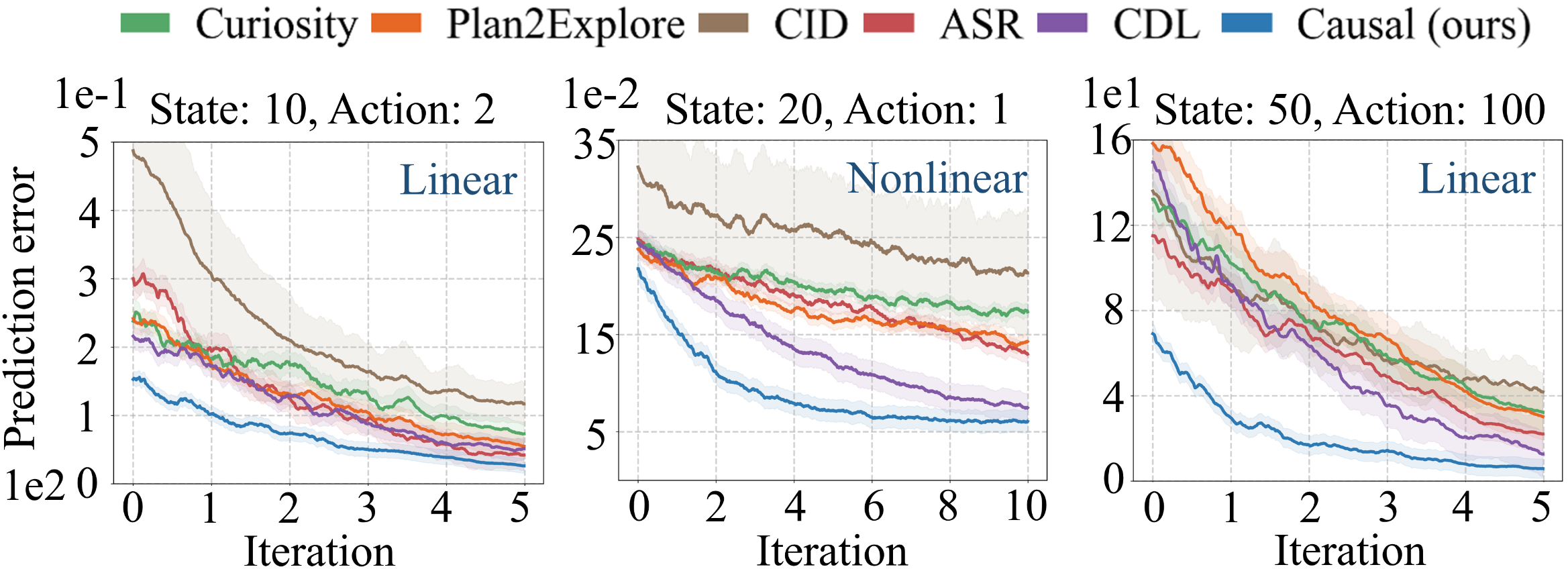}
  \caption{Prediction ability of different learned models on the synthetic environment under various state and action spaces.}
  \label{fig:synthetic_result}
\end{figure}

\paragraph{Advantage of causal-based models at the initial iteration.}
We also observe that many of the prediction error curves in Figure \ref{fig:synthetic_result} have a different $y$ intercept, suggesting that causal exploration is better even at the initial iteration. This is because the integration of causal matrix $D$ into the model exclusively incorporates the parent nodes and eliminates extraneous input information when predicting future states, a departure from prior methodologies where these nodes were treated and initialized on par with parent nodes. We note that this advantage is also tied to the correctness of the identified causal matrix and the sparsity of the causal graph as in Theorem \ref{thm}. 

\paragraph{Ablation study and generalization to some challenging scenarios.}
We further perform ablation studies: For online causal discovery, we provide results with various sampling methods; for the optimization of exploration strategy, we first investigate the impact of active reward and then perform causal exploration with some other forms of intrinsic reward. Besides, since the reliability of identified causal structures is an essential guarantee for the performance of causal exploration, we also consider challenging scenarios where the agent starts with wrong causal beliefs and encounters sudden structural changes. All of these corresponding experimental results, along with the implementation details, are given from Appendix \ref{ape:abl_stu} to Appendix \ref{ape:mujoco}.

\subsection{Real World Applications}
\paragraph{Application to Traffic Signal Control Task.}
We further conduct experiments in the challenging traffic signal control task to evaluate our proposed method. Table \ref{tab:performance} lists the performance of different learned world models on downstream policy learning. The reward here is a combination of several terms including the sum of queue length and so on, defined in Equation $(3)$ in IntelliLight (\cite{wei-intellilight-2018}). Duration refers to averaged travel time vehicles spent on approaching lanes (in seconds). Length is the sum of the length of waiting vehicles over all approaching lanes and Vehicles is the total number of vehicles that passed the intersection. We also train an agent in the true environment to verify the reliability of the learned models through task-agnostic exploration. The performance of the policy learned on the world models trained during causal exploration is comparable to, or even in some metrics better than, that of the true environment. Appendix \ref{ape:traffic} provides the detailed descriptions, including the causal graph on state variables in traffic signal control task, and comparison of the prediction errors during exploration.
\begin{table}[tbp]
  \centering
  \resizebox{\linewidth}{!}{
  \begin{tabular}{lcccc}
  \hline
   Model &  Reward &  Duration &  Length &  Vehicles \\
  \hline
   Curiosity &  -3.51 &  10.92 &  1.55 &  304 \\
   CID &  -4.36 &  12.45 &  3.62 &  288 \\
   ASR &   -1.69 &  8.23 &  0.92 &  312 \\
   CDL &  -2.55 &  9.39 &  1.32 &  308 \\
   \textbf{Ours} &  \textbf{-1.37} &  \textbf{7.21} &  \textbf{0.31} &  \textbf{314} \\
  \hline
  \rowcolor{gray!25} \textit{True Environment} &  \textit{-1.01} &  \textit{7.61} &  \textit{0.11} &  \textit{316} \\
  \hline
  \end{tabular}}
  \caption{Performances of downstream policy learning in traffic signal control task.}
  \label{tab:performance}
\end{table}

\paragraph{Application to MuJoCo Tasks.}
\begin{figure}[tbp]
  \centering
	\includegraphics[width=.95\linewidth]{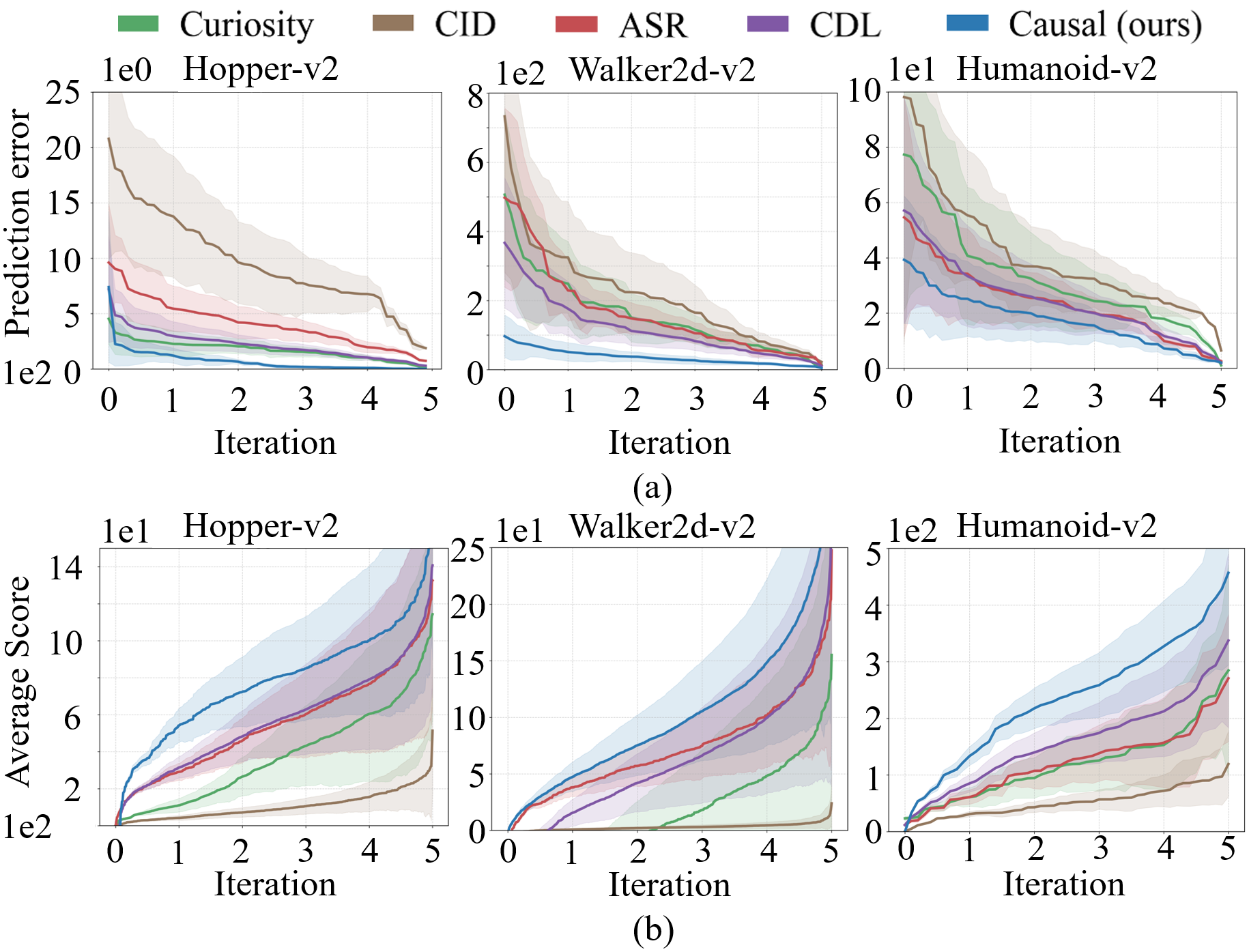}
  \caption{(a) Prediction errors in world model learning and (b) downstream policy learning performance on MuJoCo tasks.}
  \label{fig:MuJoCo}
\end{figure}
We also evaluate causal exploration on the challenging MuJoCo tasks, where the state-action dimensions range from tens (Hopper-v2) to hundreds (Humanoid-v2). Implementation details and more experimental results including the identified causal structures are given in Appendix \ref{ape:mujoco}. As shown in Figure \ref{fig:MuJoCo}(a), the world model learned during causal exploration keeps performing lower prediction errors than others. This ability of our causal method stems from its focus on causal structural constraints. By selectively incorporating parent nodes during state prediction, our approach maintains informed exploration even when exploratory signals decline, which prevents the agent from getting stuck in sub-optimal behaviors. Such a good performance is also achieved in a shorter time. Task performance in Figure \ref{fig:MuJoCo}(b) also reflects the reliability of these world models. We see that our proposed method still works well when the state space is large. As long as there is a strong causal relationship between the observed state variables, and the causal discovery algorithm can accurately identify the corresponding causal structure, our causal exploration method is believed to be well applied to high-dimensional situations.

\section{Conclusion and Future Work}
In this paper, we introduce causal exploration, a methodology designed to incorporate causal information from data for the purpose of learning world models efficiently. In particular, we employ causal exploration within the domain of task-agnostic reinforcement learning and design a sharing-decomposition schema to leverage causal structural knowledge for the world model. A series of experiments in both simulated environments and real-world tasks demonstrate the superiority of causal exploration, which highlights the importance of rich causal prior knowledge for efficient data collection and model learning. We would like to point out that this study primarily focuses on scenarios where states are fully observed. Future research directions involve addressing more complex scenarios including unobserved latent state variables and developing improved fault-tolerant mechanisms to enhance robustness.  This work provides a promising direction for future endeavors in efficient exploration.

\section*{Acknowledgments}
Yupei Yang, Shikui Tu and Lei Xu would like to acknowledge the support by the Shanghai Municipal Science and Technology Major Project, China (Grant No. 2021SHZDZX0102), and by the National Natural Science Foundation of China (62172273). 

\bibliographystyle{named}
\bibliography{ijcai24}

\clearpage
\appendix
In the following, we first present the proof of Theorem \ref{thm}. After that, more experimental details about the synthetic experiments and real-world applications are given.

\section{Proof of Theorem \ref{thm}}\label{ape:proof}

Recall that we consider solving a linear regression problem under $m$ strong convex and $M$-smooth. MSE loss function $L_f(\boldsymbol{w})$ in the formulation $\| \boldsymbol{w}^\top \cdot (\boldsymbol{s}_{t-1}, \boldsymbol{a}_{t-1}) - \boldsymbol{s}_t \|^2_2 $ with $\boldsymbol{s}_t={\boldsymbol{w}^\star}^\top \cdot (\boldsymbol{s}_{t-1}, \boldsymbol{a}_{t-1})$ is minimized by gradient descent method with 
\begin{equation}\label{equ:gd}
  \boldsymbol{w}(k) = \boldsymbol{w}(k-1) - \alpha \nabla L_f(\boldsymbol{w}(k-1))
\end{equation}
where $\alpha$ is the corresponding step size.

Denote $\boldsymbol{w}^c(k)$ as the network parameters taking causal constraints with respect to $\boldsymbol{w}(k)$ that does not gather causal information. Theorem \ref{thm} shows that causal exploration gets a prediction error bound $\delta^k$ times lower at the $k$-th step, where $\delta$ is a density measurement of the causal adjacency matrix $D$.

Below, we present a two-step proof for the convergence of causal exploration. Lemma \ref{lem:convergence} provides an upper bound for convergence without using any causal information. Lemma \ref{lem:bound} demonstrates that utilizing causal structure information results in $\boldsymbol{w}^c(k)$ being closer to the optimal value $\boldsymbol{w}^\star$ compared to $\boldsymbol{w}(k)$ at the same optimization steps. Combining Lemma \ref{lem:convergence} and Lemma \ref{lem:bound}, we derive the convergence rate for causal exploration.

\begin{lem}\label{lem:convergence}
  Suppose $L_f(\boldsymbol{w})$ is $m$-strongly convex and $M$-smooth. We have 
  \begin{equation}
    \| \boldsymbol{w}(k) - \boldsymbol{w}^\star \|^2 \le (1 - \frac mM )^k \|\boldsymbol{w}_0 - \boldsymbol{w}^\star \|^2.
  \end{equation}
\end{lem}

\begin{proof}
  According to gradient descent method (\ref{equ:gd}), we get
  \begin{equation*}
    \begin{aligned}
      &\| \boldsymbol{w}(k) - \boldsymbol{w}^\star \|^2 = \| \boldsymbol{w}(k-1) - \alpha \nabla L_f(\boldsymbol{w}(k-1)) - \boldsymbol{w}^\star \|^2 \\
      &= \| \boldsymbol{w}(k-1) - \boldsymbol{w}^\star \|^2 + \alpha^2 \|\nabla L_f(\boldsymbol{w}(k-1))\|^2 \\
      &\quad - 2 \alpha \nabla L_f(\boldsymbol{w}(k-1))(\boldsymbol{w}(k-1) - \boldsymbol{w}^\star).
    \end{aligned}
  \end{equation*}

  By strong convexity
  \begin{equation}
    \nabla L_f(\boldsymbol{w})(\boldsymbol{w} - \boldsymbol{w}^\star) \ge L_f(\boldsymbol{w}) - L_f(\boldsymbol{w}^\star) + \frac m2 \|\boldsymbol{w}-\boldsymbol{w}^\star\|^2,
  \end{equation}
we further obtain
\begin{equation}
  \begin{aligned}
    &\| \boldsymbol{w}(k) - \boldsymbol{w}^\star \|^2 \le \| \boldsymbol{w}(k-1) - \boldsymbol{w}^\star \|^2 - 2 \alpha (L_f(\boldsymbol{w}(k-1)) \\
    &- L_f(\boldsymbol{w}^\star) + \frac m2 \|\boldsymbol{w}(k-1)-\boldsymbol{w}^\star\|^2) + \alpha^2 \|\nabla L_f(\boldsymbol{w}(k-1))\|^2 \\
    &= \| \boldsymbol{w}(k-1) - \boldsymbol{w}^\star \|^2 - 2 \alpha (L_f(\boldsymbol{w}(k-1)) - L_f(\boldsymbol{w}^\star)) \\ 
    &\qquad - \alpha m \|\boldsymbol{w}(k-1)-\boldsymbol{w}^\star\|^2 + \alpha^2 \|\nabla L_f(\boldsymbol{w}(k-1))\|^2 \\
    &\le \| \boldsymbol{w}(k-1) - \boldsymbol{w}^\star \|^2 - 2 \alpha (L_f(\boldsymbol{w}(k-1)) - L_f(\boldsymbol{w}^\star)) \\
    &\qquad - \alpha m \|\boldsymbol{w}(k-1)-\boldsymbol{w}^\star\|^2 \\
    &\qquad  + 2 \alpha^2 M (L_f(\boldsymbol{w}(k-1)) - L_f(\boldsymbol{w}^\star)) \\
    &\le \| \boldsymbol{w}(k-1) - \boldsymbol{w}^\star \|^2 - \alpha m \|\boldsymbol{w}(k-1)-\boldsymbol{w}^\star\|^2 \\
    &\qquad + 2 \alpha (\alpha M - 1) (L_f(\boldsymbol{w}(k-1)) - L_f(\boldsymbol{w}^\star)).
  \end{aligned}
\end{equation}
Consider $\alpha = \frac1M$, we get
\begin{equation}
  \| \boldsymbol{w}(k) - \boldsymbol{w}^\star \|^2 \le (1- \frac mM)\| \boldsymbol{w}(k-1) - \boldsymbol{w}^\star \|^2.
\end{equation}
Using the above equation repeatedly, we obtain
\begin{equation}
  \begin{aligned}
    \| \boldsymbol{w}(k) - \boldsymbol{w}^\star \|^2 &\le (1- \frac mM)\| \boldsymbol{w}(k-1) - \boldsymbol{w}^\star \|^2  \\
    &\le (1- \frac mM)^2\| \boldsymbol{w}(k-2) - \boldsymbol{w}^\star \|^2 \\
    &\le \cdots \le (1- \frac mM)^k\| \boldsymbol{w}_{0} - \boldsymbol{w}^\star \|^2.
  \end{aligned}
\end{equation}
\end{proof}

\begin{lem}\label{lem:bound}
  Suppose $\boldsymbol{w}^c(k) $ and $\boldsymbol{w}(k)$ are the network parameters with/without causal structure respectively and $\boldsymbol{w}^\star$ is the optimum. It holds that
  \begin{equation}
    \|\boldsymbol{w}^c(k) - \boldsymbol{w}^\star\|^2 \le \delta_k \|\boldsymbol{w}(k) - \boldsymbol{w}^\star\|^2.
  \end{equation}
\end{lem}

\begin{proof}
  According to Definition \ref{def:causality}, we have
  \begin{equation}
    \boldsymbol{w}^c(k) = D \odot \boldsymbol{w}(k),
  \end{equation}
  where $D$ is the binary causal matrix. Hence, we rewrite $\boldsymbol{w}(k)$ as
  \begin{equation}
    \boldsymbol{w}(k) = D \odot \boldsymbol{w}(k) + (1-D) \odot \boldsymbol{w}(k) = \boldsymbol{w}^c(k) + \boldsymbol{w}'_k.
  \end{equation}
  Note that we have $\boldsymbol{w}^\star = D \odot \boldsymbol{w}^\star$. Then we obtain
  \begin{equation}
    \begin{aligned}
      &\|\boldsymbol{w}(k) - \boldsymbol{w}^\star\|^2 = \sum_{i,j}(w_{ij}(k) - w_{ij}^\star)^2\\
      &= \sum_{i,j}\left[D_{ij} \times w_{ij}(k) - w_{ij}^\star + (1-D_{ij}) \times w_{ij}(k)\right]^2
      \\
      &= \sum_{i,j}\left[(D_{ij} \times w_{ij}(k) - w_{ij}^\star)^2 + ((1-D_{ij}) \times w_{ij}(k))^2 \right] \\
      &= \sum_{i,j}(D_{ij} \times w_{ij}(k) - w_{ij}^\star)^2 + \sum_{i,j}((1-D_{ij}) \times w_{ij}(k))^2  \\
      &= \sum_{i,j}(w^c_{ij}(k) - w_{ij}^\star)^2 + \sum_{i,j}(w'_{ij})(k)^2  \\
      &= \| \boldsymbol{w}^c(k) - \boldsymbol{w}^\star\|^2 + \|\boldsymbol{w}'_k \|^2 \ge (1 + \rho_k) \| \boldsymbol{w}^c(k) - \boldsymbol{w}^\star\|^2,
    \end{aligned}
  \end{equation}
  where $\rho_k \in [0, +\infty)$ is the lower bound of the ratio between $\|\boldsymbol{w}'_k\|^2$ and $\|\boldsymbol{w}^c(k) - \boldsymbol{w}^\star\|^2$ whose value is related to the sparsity of causal matrix $D$. By setting $\delta_k = \frac{1}{1 + \rho_k}$, we complete the proof of Lemma \ref{lem:bound}.
\end{proof}

The convexity of $L_f(\boldsymbol{w})$ implies
\begin{equation}\label{equ:convexity}
  L_f(\boldsymbol{w}^c(k)) - L_f(\boldsymbol{w}^\star) \le \frac M2 \| \boldsymbol{w}^c(k) - \boldsymbol{w}^\star \|^2.
\end{equation}

By applying Lemma \ref{lem:convergence} and Lemma \ref{lem:bound} into (\ref{equ:convexity}) and denote
\begin{equation}
  \delta = \max \{\delta_0, \delta_1, \ldots, \delta_k\},
\end{equation}
we can obtain Theorem \ref{thm}.

\section{Synthetic environment}\label{ape:abl_stu}

\subsection{More experiment details}\label{ape:syn_detail}

\begin{figure}[htbp]
    \centering
    \includegraphics[width=\linewidth]{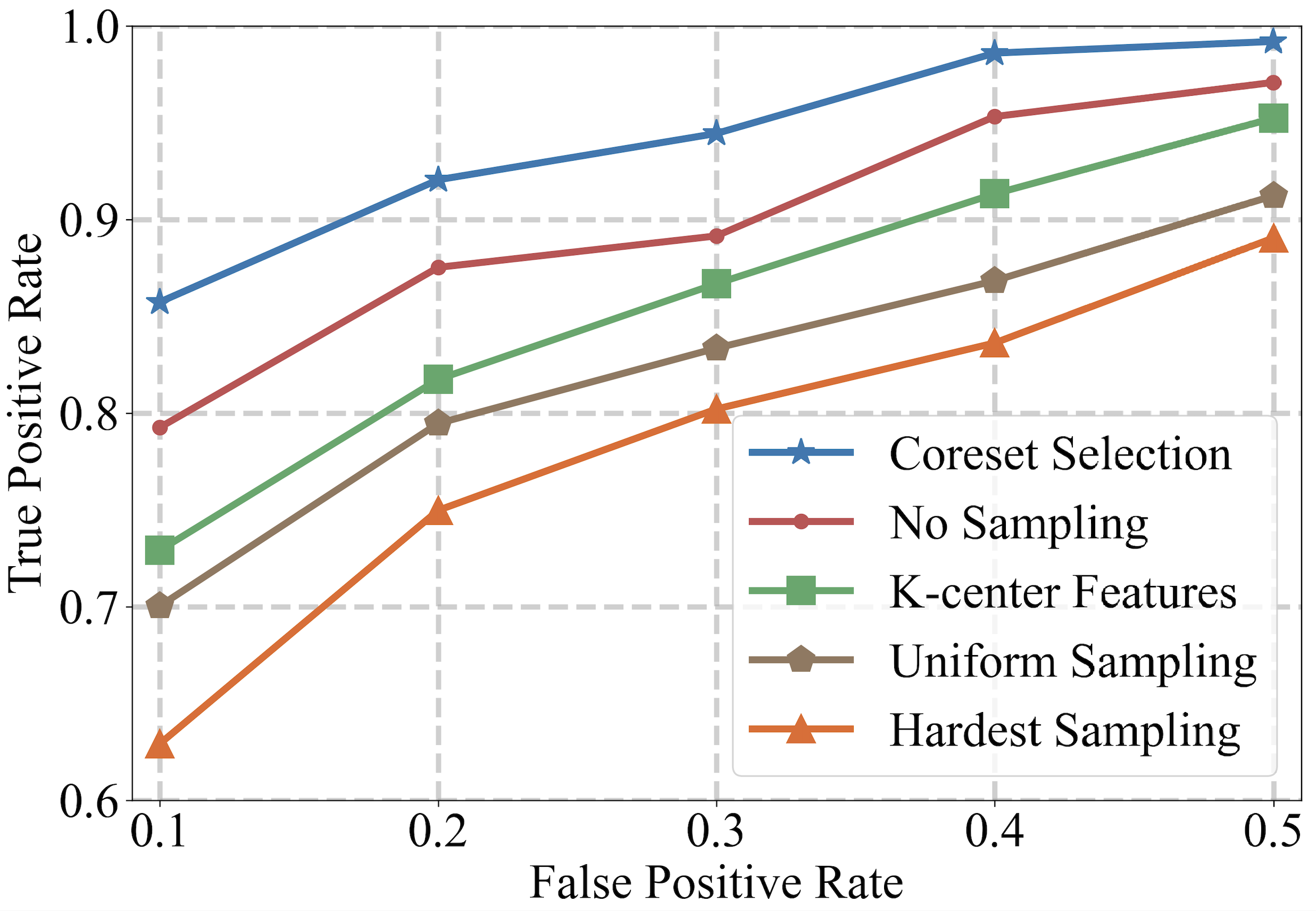}
    \caption{ROC curves of online causal discovery using different sampling methods.}
    \label{fig:online_pc}
\end{figure}
\begin{table*}[htbp]
  \centering
  \begin{tabular}{cccccc}
  \hline
  Method & No Sampling & Uniform & K-center & Hardest & Coreset \\
  \hline
  AUC & 0.907 $\pm$ 0.084 & 0.892 $\pm$ 0.065 & 0.891 $\pm$ 0.175 & 0.816 $\pm$ 0.353 & \textbf{0.969 $\pm$ 0.016} \\
  Precision & 0.953 $\pm$ 0.074 & 0.914 $\pm$ 0.079 & 0.942 $\pm$ 0.089 & 0.892 $\pm$ 0.081 & \textbf{0.970 $\pm$ 0.064} \\
  F1-score & 0.831 $\pm$ 0.094 & 0.842 $\pm$ 0.104 & 0.890 $\pm$ 0.139  & 0.852 $\pm$ 0.067 & \textbf{0.928 $\pm$ 0.096} \\
  Time & 4.360 $\pm$ 2.888 & 0.880 $\pm$ 0.453 & 0.909 $\pm$ 0.414 & 1.389 $\pm$ 0.760 & \textbf{0.379 $\pm$ 0.267} \\
  \hline
  \end{tabular}
  \caption{Results of different selection methods on online causal discovery.}
  \label{tab:kci-speed-up}
\end{table*}

\paragraph{Implementation.} Double DQN \cite{van-deep-2016} is used to train the exploration policy of agents, where both the evaluation network and the target network are three-layer fully connected networks with relu activations. The corresponding world models are designed as MLPs without/with activation function in linear/nonlinear module respectively, with 32 and 8 hidden nodes. In all of these settings, we share network parameters except for the last layer which is separately decomposed. We use Adam optimizer to learn the above models with $\eta = 0.1$ and a learning rate of 1e-3. 

\paragraph{The $\kappa$ value for online causal discovery.}
\begin{figure}
    \centering
    \includegraphics[width=\linewidth]{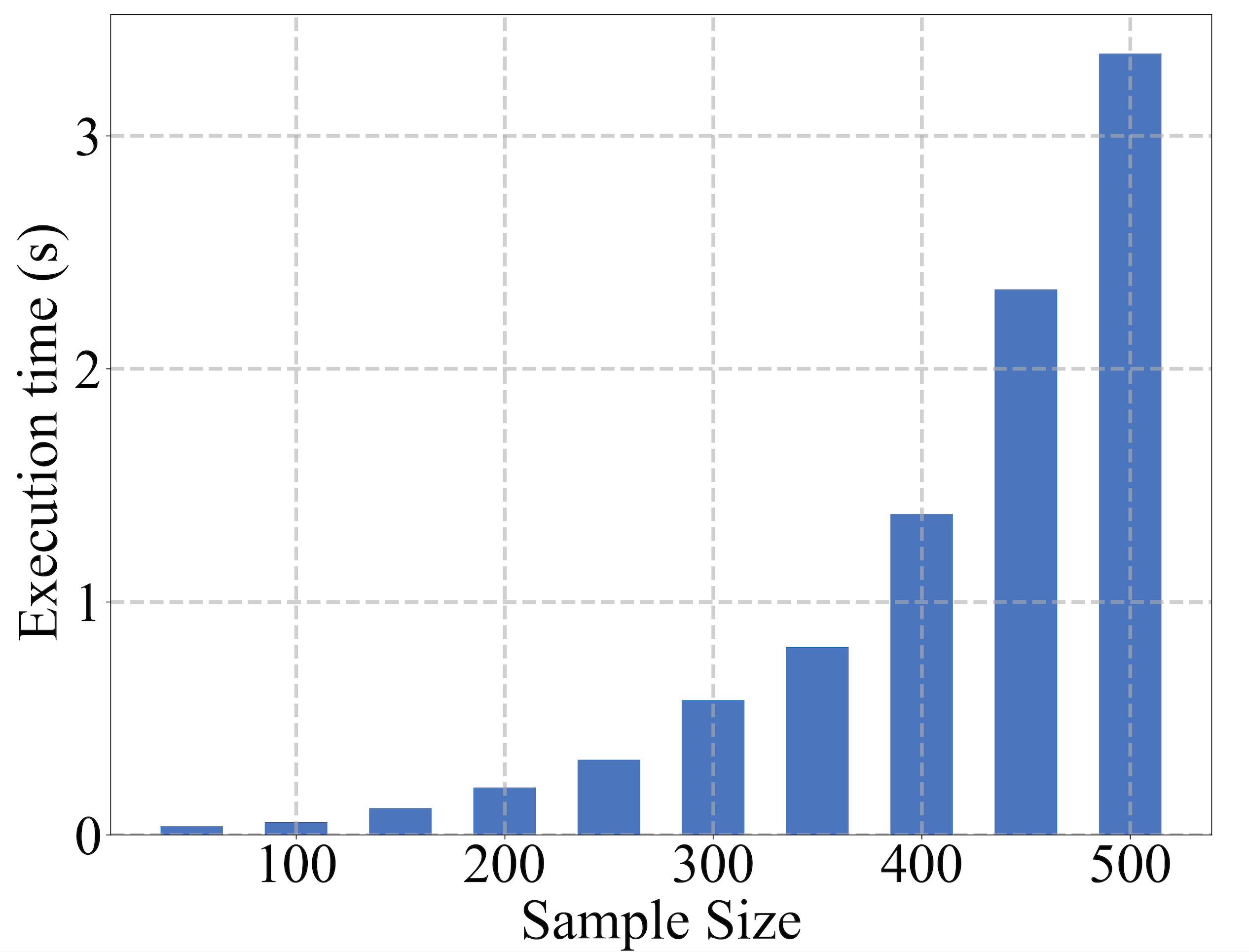}
    \caption{Time cost for the PC algorithm when $|\mathcal{U}| = 12$ and $|\mathcal{V}|=10$.}
    \label{fig:pc-time}
\end{figure}
As can be seen in Figure \ref{fig:pc-time}, the execution time of the PC algorithm exhibits a noticeable inflection when the sample size is around $350$. This provides a valuable reference for determining an appropriate value for the selection number $\kappa$. As a result, we empirically set $\kappa=350$ in the synthetic environment and $\kappa=0.7 \times |\mathcal{B}_t|$ for real-world applications, where $|\mathcal{B}_t|$ represents the number of collected data at time $t$.

\paragraph{Efficient Causal Discovery.} To speed up PC algorithm with KCI-test, we design an efficient online causal discovery using selection methods. Figure \ref{fig:online_pc} illustrates the corresponding ROC curves of different selection methods including Uniform Sampling, K-center Features \cite{nguyen-variational-2017}, Hardest Sampling \cite{Aljundi-Online-2019} and the Coreset Selection method \cite{yoon-online-2021} we use. Tables \ref{tab:kci-speed-up} summarizes the performance of different sampling methods in our synthetic environments. These exciting experimental results demonstrate the superiority of the Coreset Selection method we used in improving the efficiency of causal discovery.

\paragraph{Causal Exploration Experiments.} The causal matrix in our synthetic environment is set as a lower triangular matrix whose elements are generated from the uniform distribution $[-8, 8]$. After that, for each edge in the graph, we randomly drop it out with an empirically chosen probability $1 - p = 0.8$. Besides, the covariance $\Sigma$ is a diagonal matrix whose elements are randomly generated from $[0, 0.1]$. For both linear and nonlinear conditions, we have conducted sufficient experiments. Figure \ref{fig:appendix_synthtic} shows the remaining experimental results that are not fully presented in the main body. 

In order to enhance the agent's sensitivity to causal informative data, we design a novel form of active reward. Here we investigate the impact of $\beta$ in this new intrinsic reward formulation on causal exploration which is formulated as $r_t = r^i_t + \beta r^a_t$ in equation (\ref{equ:active}). Figure \ref{fig:beta} illustrates the evolution of prediction error over training time for different values of $\beta$. We see that incorporating active reward for exploration continuously improves the performance of the world model. We set $\beta=0.5$ for linear environments and $\beta=3$ for both nonlinear settings and real-world applications according to the results in Figure \ref{fig:beta}, respectively. 

\begin{figure*}[tbp]
  \centering
    \includegraphics[width=.9\linewidth]{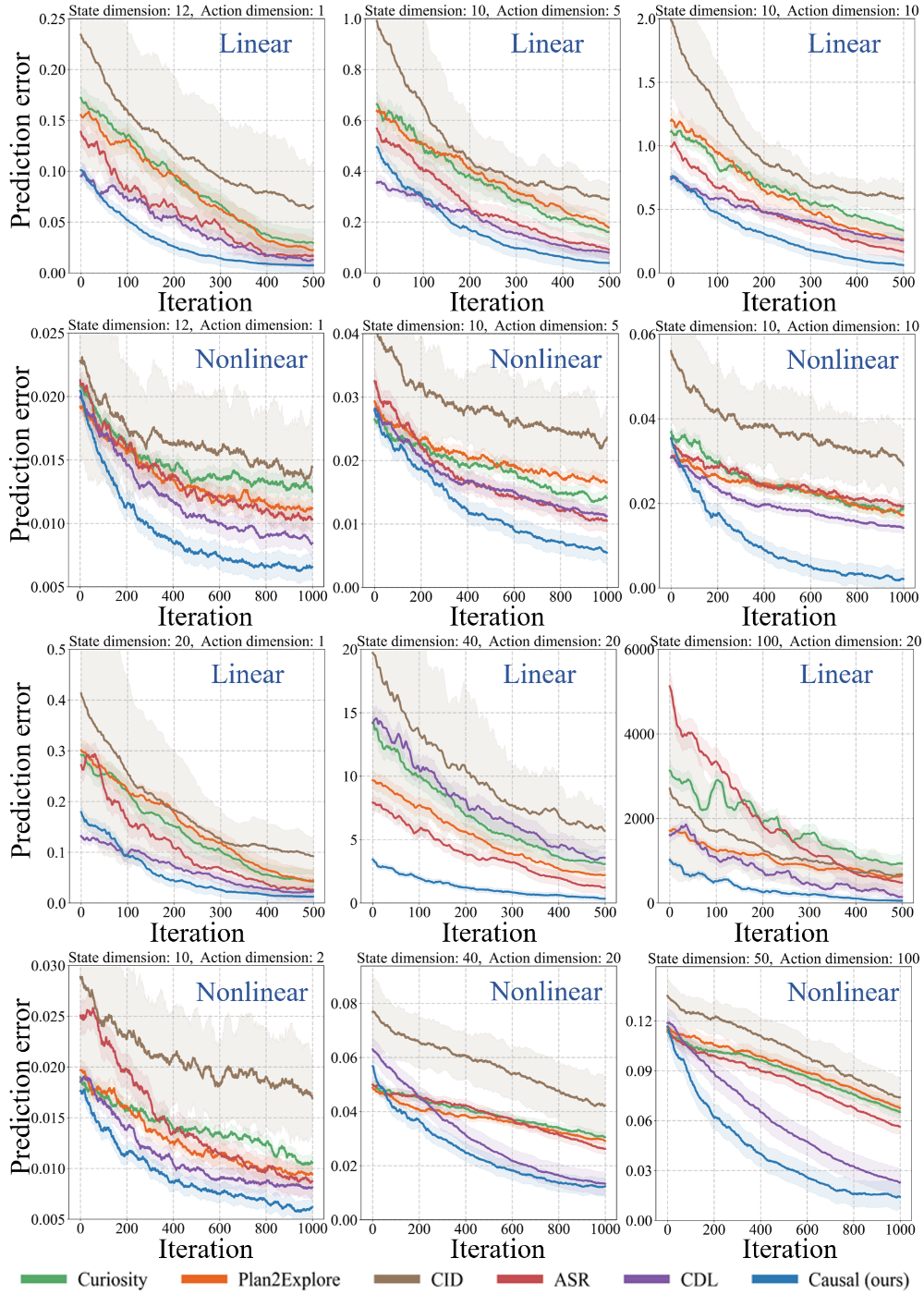}
  \caption{Results on synthetic datasets.}
  \label{fig:appendix_synthtic}
\end{figure*}

\begin{figure*}
  \centering
	\subfigure[]{
	\includegraphics[width=.9\linewidth]{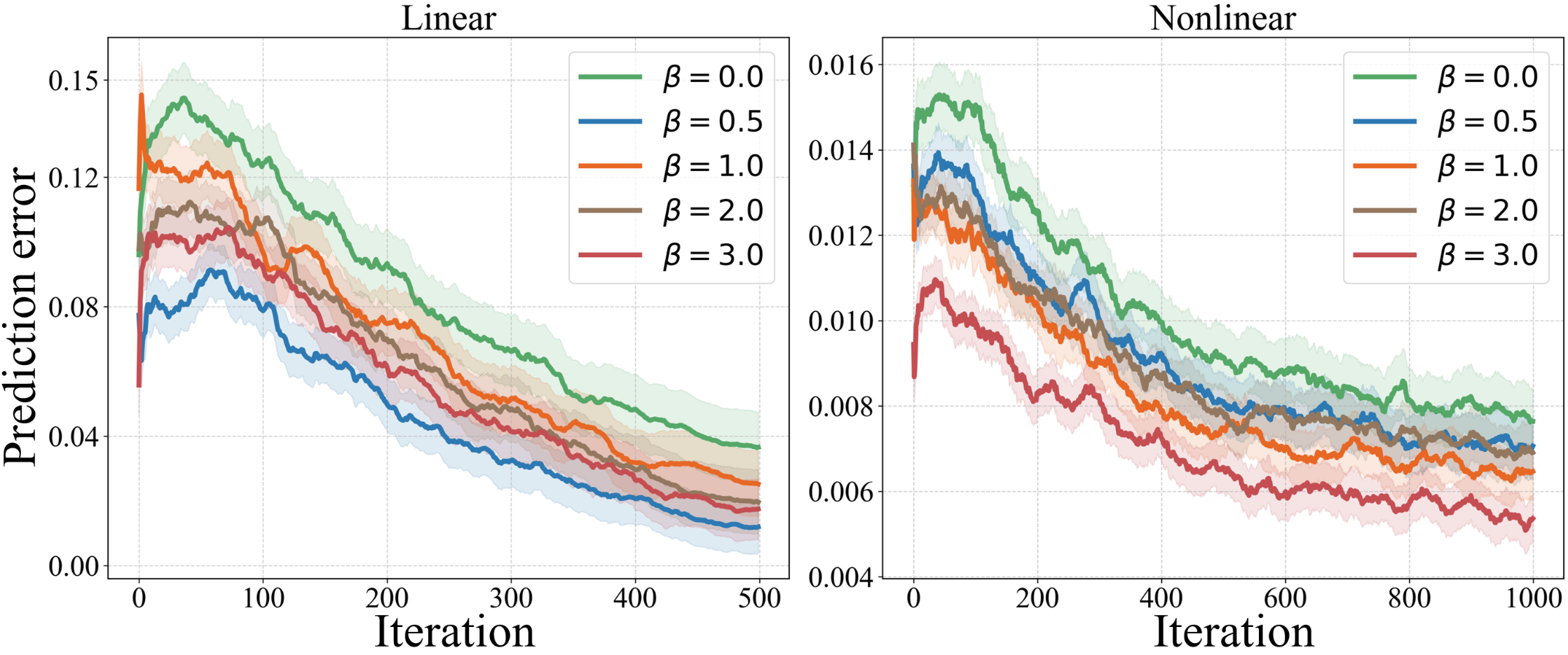} \label{fig:beta}
}
	\subfigure[]{
	\includegraphics[width=.9\linewidth]{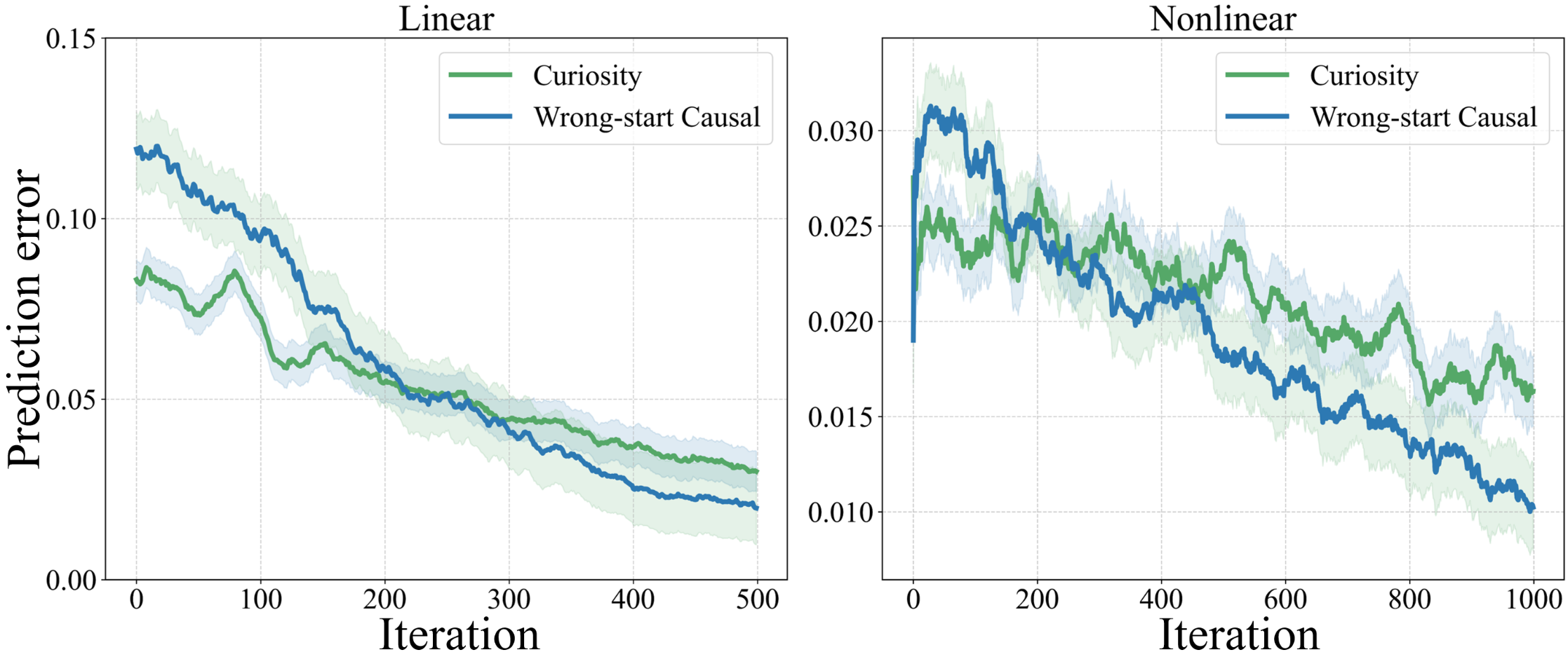} \label{fig:underestimation}
}
	\subfigure[]{
	\includegraphics[width=.9\linewidth]{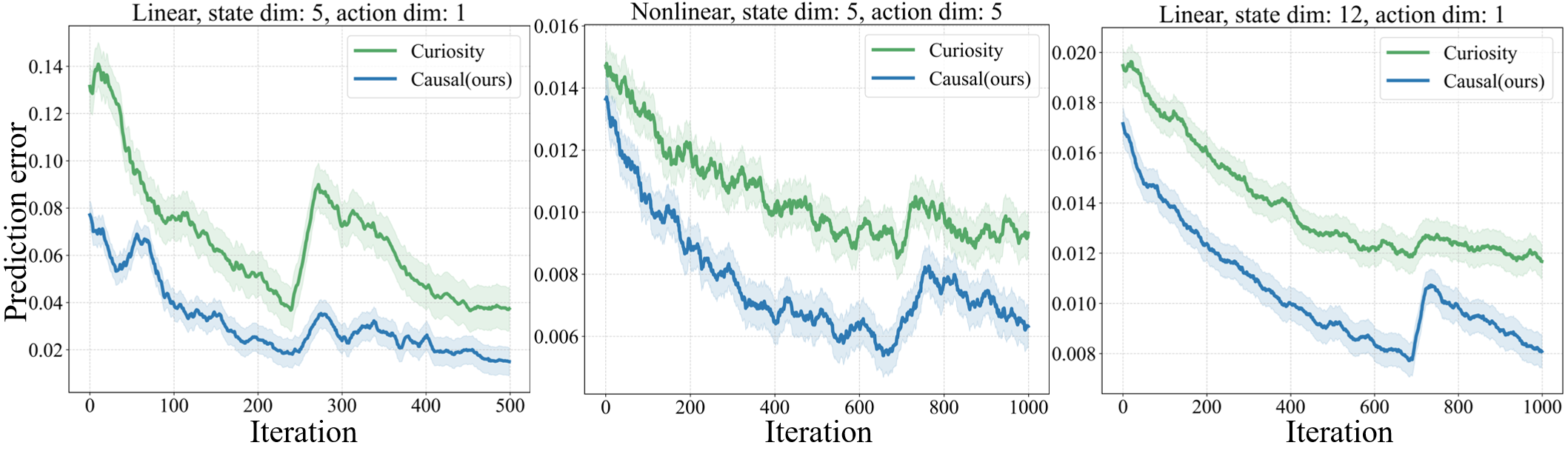} \label{fig:correct}
}
\caption{(a) Prediction errors of causal exploration for different $\beta$; (b) Performance on underestimation scenarios; (c) Scenarios with structural changes.}
\end{figure*}

\begin{figure*}
  \centering
	\subfigure[]{
    \resizebox{.65\linewidth}{0.3\linewidth}{
    \includegraphics[width=0.3\linewidth]{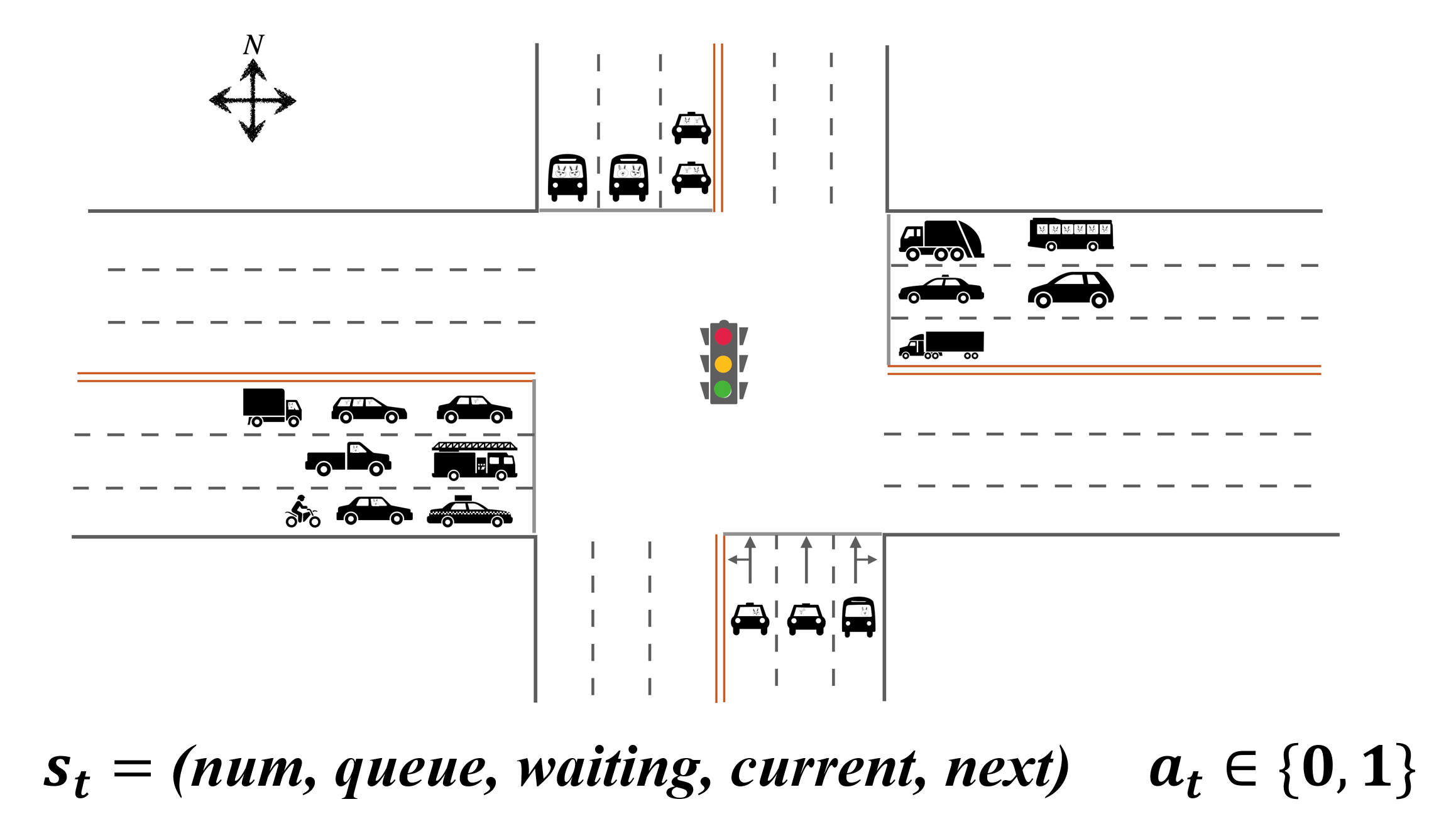} \label{fig:traffic_illu}
    }
}
	\subfigure[]{
    \resizebox{.15\linewidth}{.3\linewidth}{
    \begin{tikzpicture}[mycircle/.style={circle,draw=black!50,fill=white!20,very thick, minimum size=15mm, inner sep=0pt, font=\fontsize{18}{18}\selectfont},myfirstnode/.style={circle,draw=black!30,fill=black!20,very thick, minimum size=15mm, inner sep=0pt, font=\fontsize{18}{18}\selectfont}]
        \node (n1) at (0, 4) [mycircle] {$s_{1, t}$};
        \node (n2) at (0, 2)[mycircle] {$s_{2, t}$};
        \node (n3) at (0, 0)[mycircle] {$s_{3, t}$};
        \node (n4) at (0, -2)[mycircle] {$s_{4, t}$};
        \node (n5) at (0, -4)[mycircle] {$s_{5, t}$};
        \node (s1)at ($(n3)+(-3.5,5)$)  [mycircle] {$s_{1, t-1}$} edge[->, bend angle=0, draw=black!70, >=stealth, line width=2.5pt](n1) edge[->, bend angle=0, draw=black!70, >=stealth, line width=2.5pt](n3);
        \node (s2)at ($(n3)+(-3.5,3)$)  [mycircle] {$s_{2, t-1}$} edge[->, bend angle=0, draw=black!60, >=stealth, line width=2.5pt](n2) edge[->, bend angle=0, draw=black!60, >=stealth, line width=2.5pt](n3);
        \node (s3)at ($(n3)+(-3.5,1)$)  [mycircle] {$s_{3, t-1}$} edge[->, bend angle=0, draw=black!70, >=stealth, line width=2.5pt](n2) edge[->, bend angle=0, draw=black!70, >=stealth, line width=2.5pt](n3);;
        \node (s4)at ($(n3)+(-3.5,-1)$) [mycircle] {$s_{4,t-1}$}  edge[->, bend angle=0, draw=black!70, >=stealth, line width=2.5pt](n4);
        \node (s5)at ($(n3)+(-3.5,-3)$)  [mycircle] {$s_{5,t-1}$} edge[->, bend angle=0, draw=black!70, >=stealth, line width=2.5pt](n1) edge[->, bend angle=0, draw=black!70, >=stealth, line width=2.5pt](n2)  edge[->, bend angle=0, draw=black!70, >=stealth, line width=2.5pt](n3) edge[->, bend angle=0, draw=black!80, >=stealth, line width=2.5pt](n5);
        \node (s6)at ($(n3)+(-3.5, -5)$) [myfirstnode] {$a_{t-1}$} edge[->, bend angle=0, draw=blue!70, >=stealth, line width=2.5pt](n4);
    \end{tikzpicture}
    }
    \label{fig:traffic_causal}
}
\caption{(a) Illustration of the traffic-signal-control environment. (b) Causal graph learned through exploration.}
\end{figure*}

\subsection{Generalization to underestimation scenarios}\label{ape:wrong}
In some data-hungry scenarios, there may be insufficient data for causal discovery, leading to underestimation of the causal structure, which makes continuous data collection and causal structure correction important components. In other words, the causal structure inferred from causal discovery algorithms may deviate from the ground truth ones, which is particularly prone to occur under conditions of limited sample size or during the initial stages of exploration. Consequently, we conduct an evaluation of our proposed algorithm's performance when the estimated causal structure exhibits insufficiencies or redundancies, a scenario termed "underestimation". To highlight the advantages of our sharing-decomposition schema in addressing such problems, we deliberately provide the agent with completely wrong causal information. An erroneous causal graph $\mathcal{G}'$ is supplied to the agent during the initial time steps, namely $t < N$ where $N$ is the period for causal discovery. We introduce perturbations to the true graph $\mathcal{G}$ by randomly eliminating or adding edges with a probability of $p'=0.8$, generating $\mathcal{G}'$. Figure \ref{fig:underestimation} shows the corresponding performance of causal exploration on synthetic data. Indeed, the agent exhibits an impressive ability to correct its causal exploration trajectory in the face of misdirection. Nonetheless, it is worth noting that such a schema only serves as a technique to mitigate the impact of underestimation. The reliability of causal discovery algorithms is the essential guarantee for causal exploration.

\subsection{Generalization to scenarios with causal structural changes}\label{ape:causal_change}
In real-world scenarios, causal structure between variables can often change due to sudden disturbance. For instance, causal relationships between economic variables like stock prices, interest rates, and inflation can be subject to rapid changes caused by market crashes or policy changes.

To evaluate the effectiveness of our approach in handling such mutation, we conduct experiments in a scenario where the causal structure changes randomly once. We use our simulation model to generate the data and compare our method to a non-causal approach. Figure \ref{fig:correct} illustrates the advantages of our approach in tackling such a challenging task.

Our sharing-decomposition schema enables the agents to quickly adapt to structural changes and make appropriate adjustments. This also demonstrates the robustness of our method, which allows for timely correction of errors in the causal structure. By sharing the same decomposition modules across different time steps and tasks, our method can effectively leverage previous knowledge and transfer it to new situations, while also being flexible enough to accommodate changes in the causal structure. In addition, the ability to adapt to changing causal structures can improve the generalization ability of our method, making it more applicable to a wider range of real-world tasks. 

In our future research, we plan to expand our work to situations where changes occur within the model. In these cases, during the model learning phase, it becomes crucial to effectively detect these changes and promptly update the model. Additionally, when it comes to policy learning, a key challenge is determining the most suitable model to utilize. We may encounter entirely new models that have not been encountered before, adding an additional layer of complexity to our research.

\section{Traffic Signal Control}\label{ape:traffic}
Traffic signal control is an important means of mitigating congestion in traffic management. Compared to using fixed-duration traffic signals, an RL agent learns a policy to determine real-time traffic signal states based on current road conditions. The state observed by the agent at each time consists of five dimensions of information, namely the number of vehicles, queue length, average waiting time in each lane plus current and next traffic signal states. Action here is to decide whether to change the traffic signal state or not. For example, suppose the traffic signal is red at time $t$, if the agent takes action $1$, then it will change to green at the next time $t+1$, otherwise, it will remain red. Following the work in IntelliLight \cite{wei-intellilight-2018}, the traffic environment in our experiment is a three-lane intersection. Table \ref{tab:dataset} gives a detailed description, and Figure \ref{fig:traffic_illu} provides an illustration. The estimated causal structures are given in Figure \ref{fig:traffic_causal}.

\paragraph{Experiment details and analysis.}  We first only use prediction-based causal exploration to learn forward dynamic world models under the same traffic environment in IntelliLight. Then, the agent learns a policy for traffic signal control task in our learned world models, which avoids the high-cost interaction with real traffic environment.  For consistency and easy comparison, we use the same DQN network from IntelliLight to train our causal exploration agent.

\begin{figure*}[tbp]
  \centering
  \includegraphics[width=.8\linewidth]{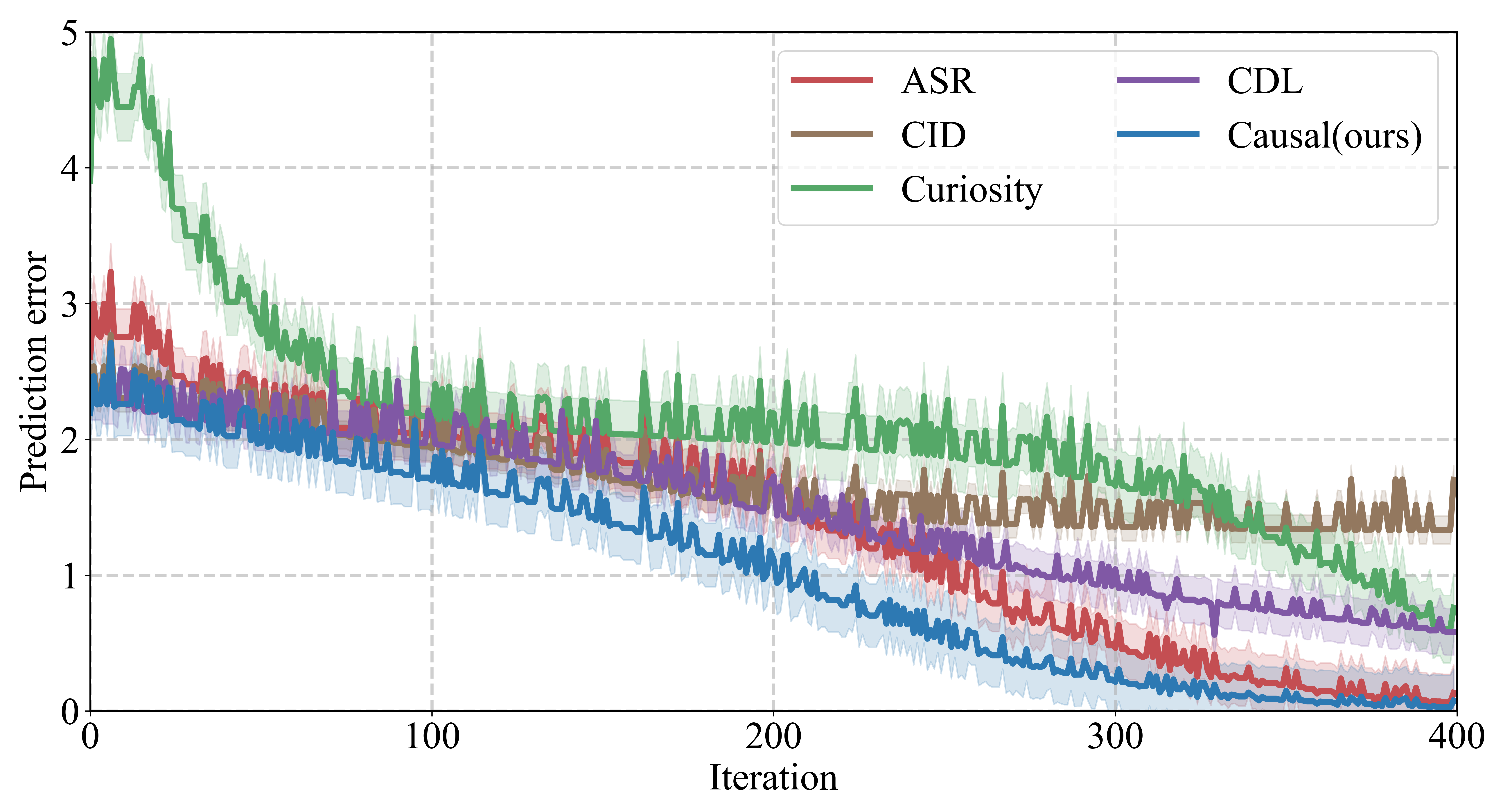} 
  \caption{Prediction errors in traffic light control.}
  \label{fig:traffic_pred}
\end{figure*}

After that, we solve the traffic signal control task in our world model with no environment interaction in a zero-shot manner. Figure \ref{fig:traffic_pred} visualizes the prediction errors of different models during the exploration process. The corresponding causal graph is shown in Figure \ref{fig:traffic_causal}. As is illustrated, the state of the traffic signal at the next time step $s_{4,t}$ is causally linked to the previous state $s_{4, t-1}$ and action $a_{t-1}$, which is in line with the definition of traffic signal control tasks. The queue length $s_{2,t}$ is determined by previous queue length $s_{2,t-1}$ and waiting time $s_{3, t-1}$ plus the traffic state $s_{5, t-1}$. Factors influencing the waiting time $s_{3,t}$ include the number of vehicles $s_{1,t-1}$ and the queue length $s_{2, t-1}$. These results align well with the common-sense reasoning.
\begin{table}[tbp]
  \centering
  \resizebox{\linewidth}{!}{
  \begin{tabular}{ccccc}
  \hline
  \multirow{2}{*}{Traffic flow setting} & \multirow{2}{*}{Directions} & \multicolumn{2}{c}{Arrival Rate (cars/s)} & \multirow{2}{*}{Duration ($\times 10^3$ s)} \\
  \cline{3-4}
    &   & Mean & Std &  \\
  \hline
  \multirow{2}{*}{Complex traffic} & East-West & $0.211$ & $0.023$ & \multirow{2}{*}{216} \\
   & South-North & $0.155$ & $0.030$ &   \\
  \hline
  \end{tabular}}
  \caption{Traffic Dataset Description}
  \label{tab:dataset}
\end{table}

\section{More Results of Mujoco tasks}\label{ape:mujoco}
We use PPO algorithm \cite{schulman-proximal-2017} for optimization during both the task-agnostic exploration and policy learning stages and adopt the hyperparameters from Table 3 of PPO with a trajectory length of $2048$, an Adam stepsize of 3e-4, a minibatch size of 64, a discount factor ($\gamma=0.99$), a GAE parameter ($\lambda=0.95$), and a clipping parameter ($\epsilon=0.2$). Both the actor-critic network and the world model are 2-(hidden)-layer neural networks, consisting of 256 and 64 hidden nodes respectively. Activation functions are Tanh and ReLU here.

Performance of causal exploration on some other MuJoCo tasks are provided in Figure \ref{fig:mujuco}. Predictions given by world models under causal structural constraints are more accurate and stable than those of other methods. The learned world model of causal exploration provides the agent with more information in the following policy learning stage, resulting in higher scores achieved in a shorter time. Figure \ref{fig:mujuco_graph} illustrates the identified causal structures during exploration, which explains for the performance gain.
\begin{figure}[htbp]
  \centering
    \includegraphics[width=.9\linewidth]{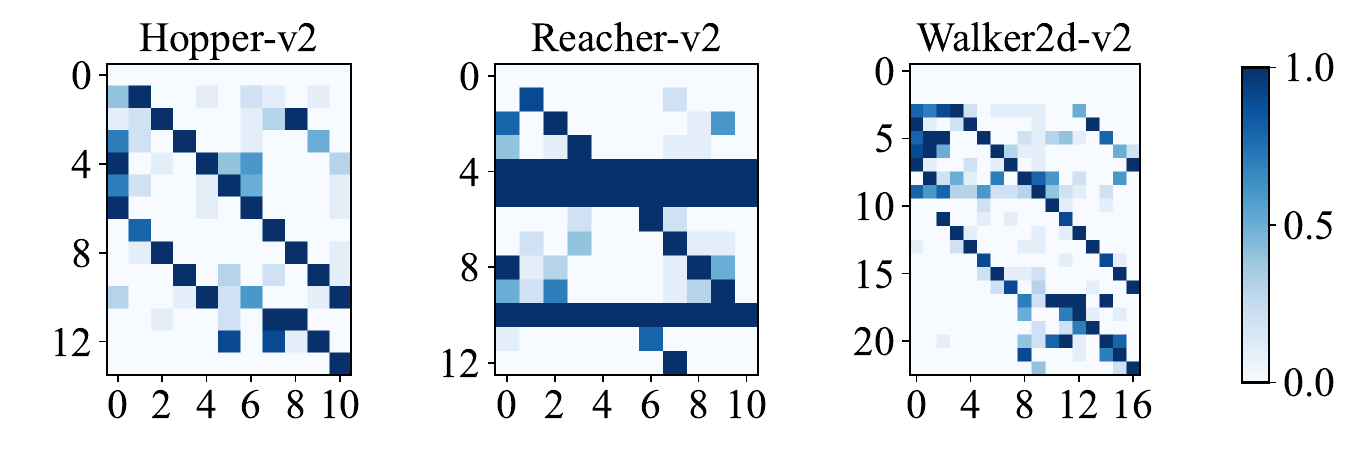}
  \caption{Identified causal structures for MuJoCo tasks.}
  \label{fig:mujuco_graph}
\end{figure}

We also conduct several experiments to test the performance of causal exploration with a different form of intrinsic reward. To be specific, we formulate our world model as $\mu_{w^c}(\boldsymbol{s}_t, \boldsymbol{a}_t), \sigma^2_{w^c}(\boldsymbol{s}_t, \boldsymbol{a}_t)$ to model the transition probability as $p(\boldsymbol{s}_{t+1} \mid \boldsymbol{s}_t, \boldsymbol{a}_t) \sim \mathcal{N}(\boldsymbol{s}_{t+1}; \mu_{w^c}, \sigma^2_{w^c})$. Then, the negative log-likelihood is used both for the world model learning and causal exploration, which is a replacement for equation (\ref{equ:pred_loss}) and (\ref{equ:intrinsic}), and is formulated as:
\begin{equation}
\begin{aligned}
    L_{(\mu_{w^c}, \sigma^2_{w^c})} &= \frac{(\boldsymbol{s}_{t+1} - \mu_{w^c}(\boldsymbol{s}_t, \boldsymbol{a}_t))^2}{2\sigma^2_{w^c}(\boldsymbol{s}_t, \boldsymbol{a}_t)} + \frac12 \log \sigma^2_{w^c}(\boldsymbol{s}_t, \boldsymbol{a}_t), \\
    r^i_t &= \frac \eta 2 L_{(\mu_{w^c}, \sigma^2_{w^c})}.
\end{aligned}
\end{equation}
Corresponding results are shown in Figure \ref{fig:mujoco_change_model}. However, various forms of intrinsic rewards don't exhibit significant differences in performance. In some tasks, the introduction of an additional covariance network even lead to performance not as favorable as when directly using regression loss.

\begin{figure*}[htbp]
  \centering
  \subfigure[]{
	\includegraphics[width=.8\linewidth]{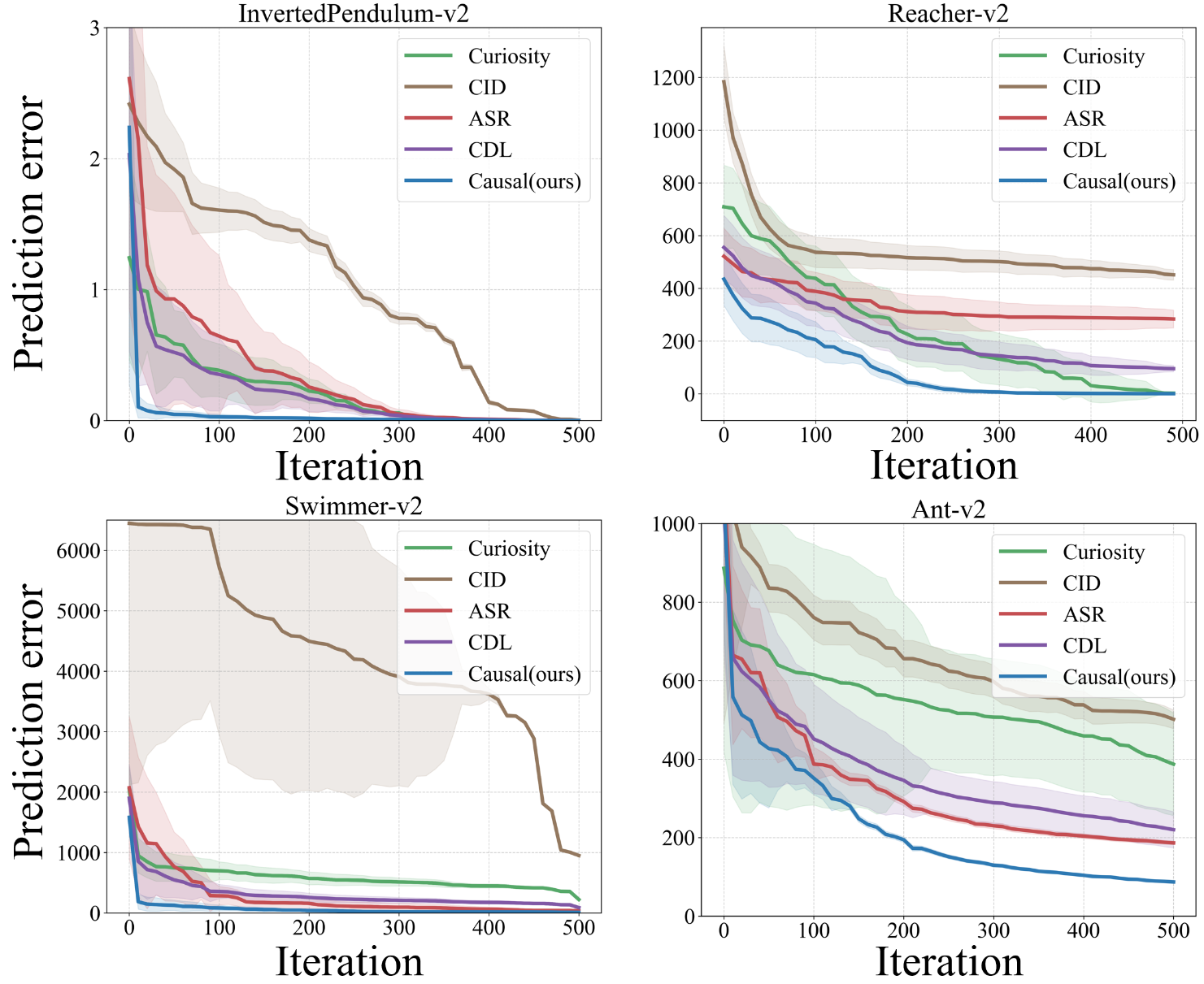}
}
  \subfigure[]{
	\includegraphics[width=.8\linewidth]{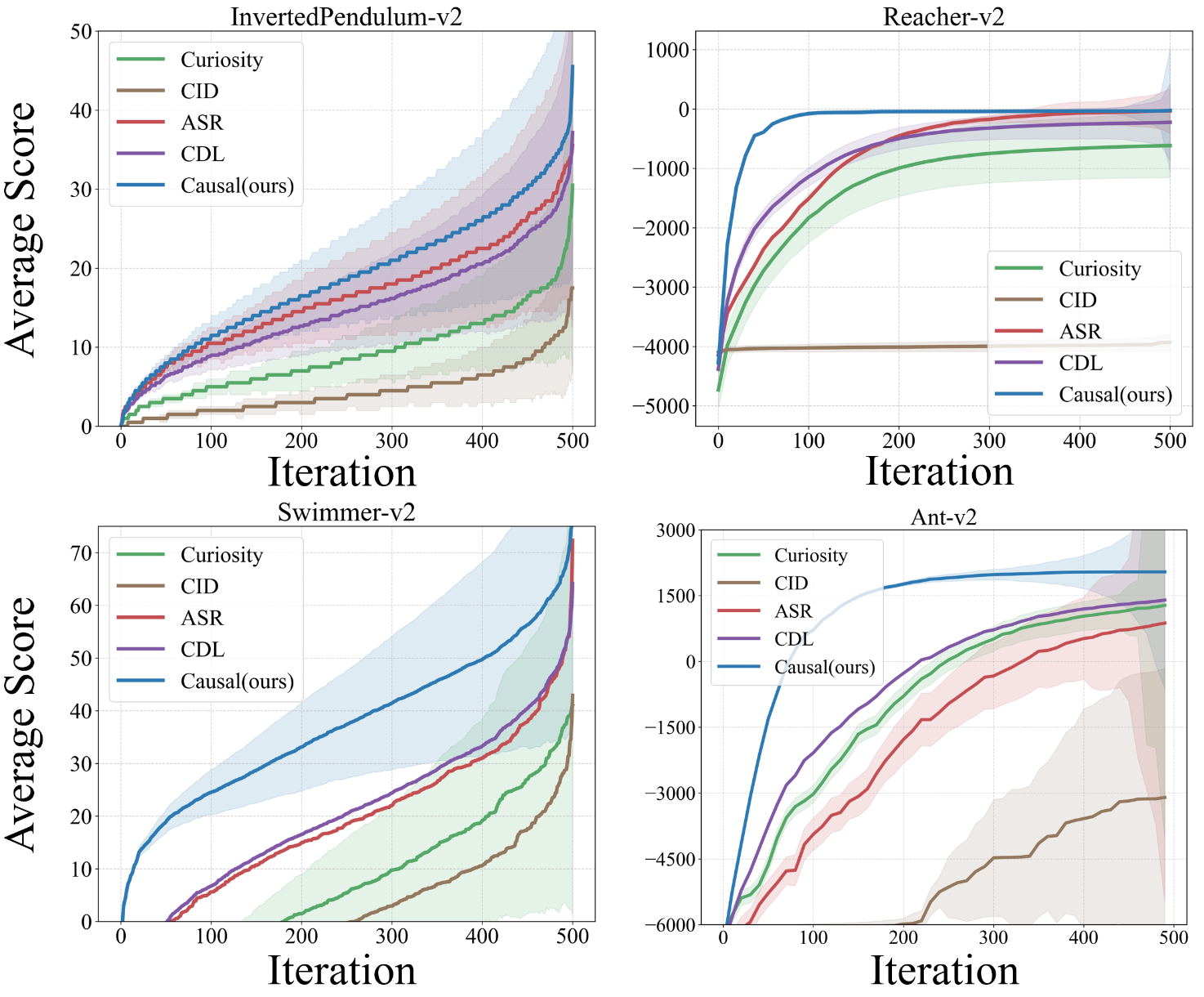}
}
  \caption{Application to some other MuJoCo tasks.}
  \label{fig:mujuco}
\end{figure*}

\begin{figure*}[htbp]
    \centering
    \subfigure[]{
	\includegraphics[width=.9\linewidth]{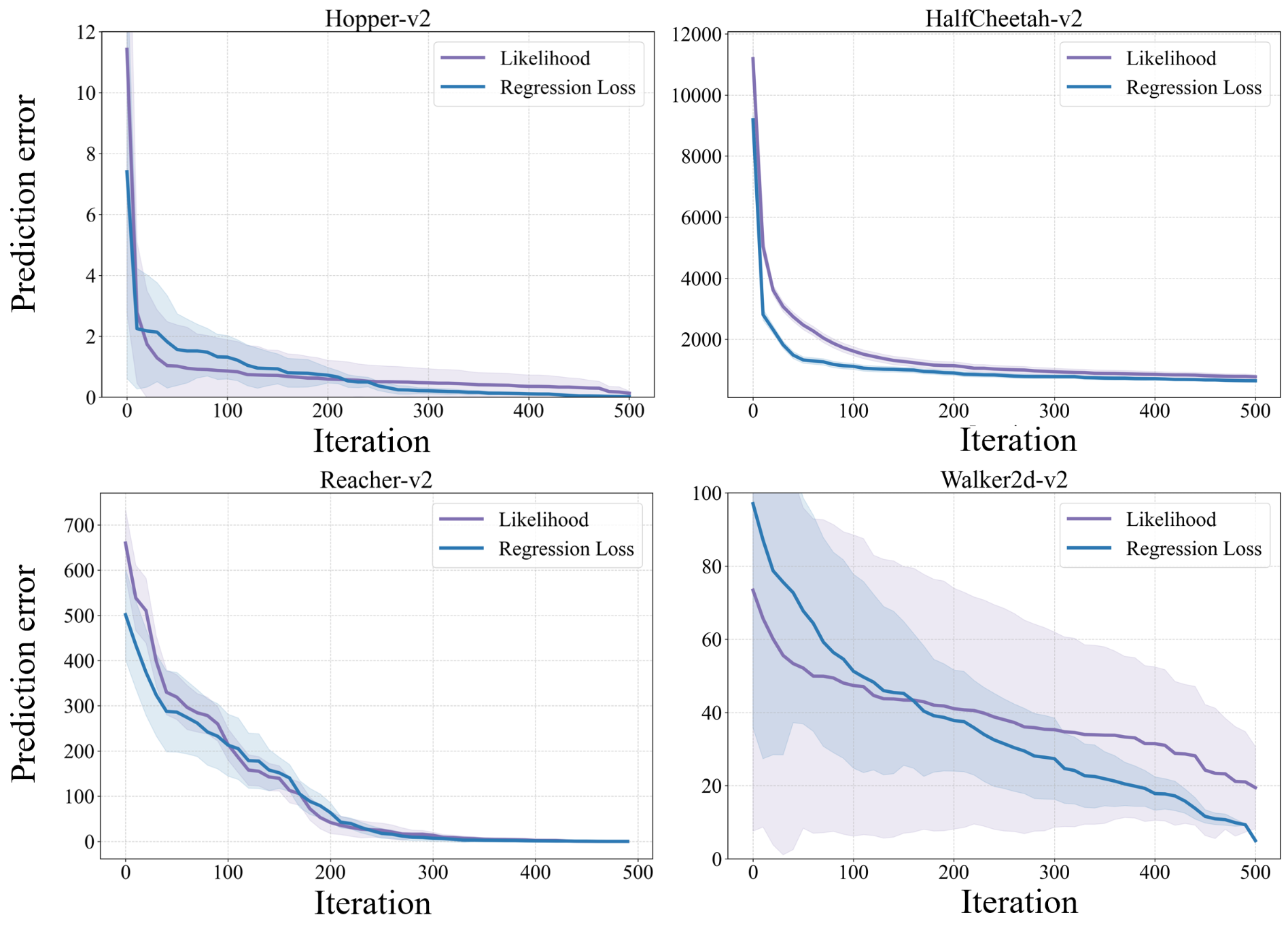}
}
  \subfigure[]{
	\includegraphics[width=.9\linewidth]{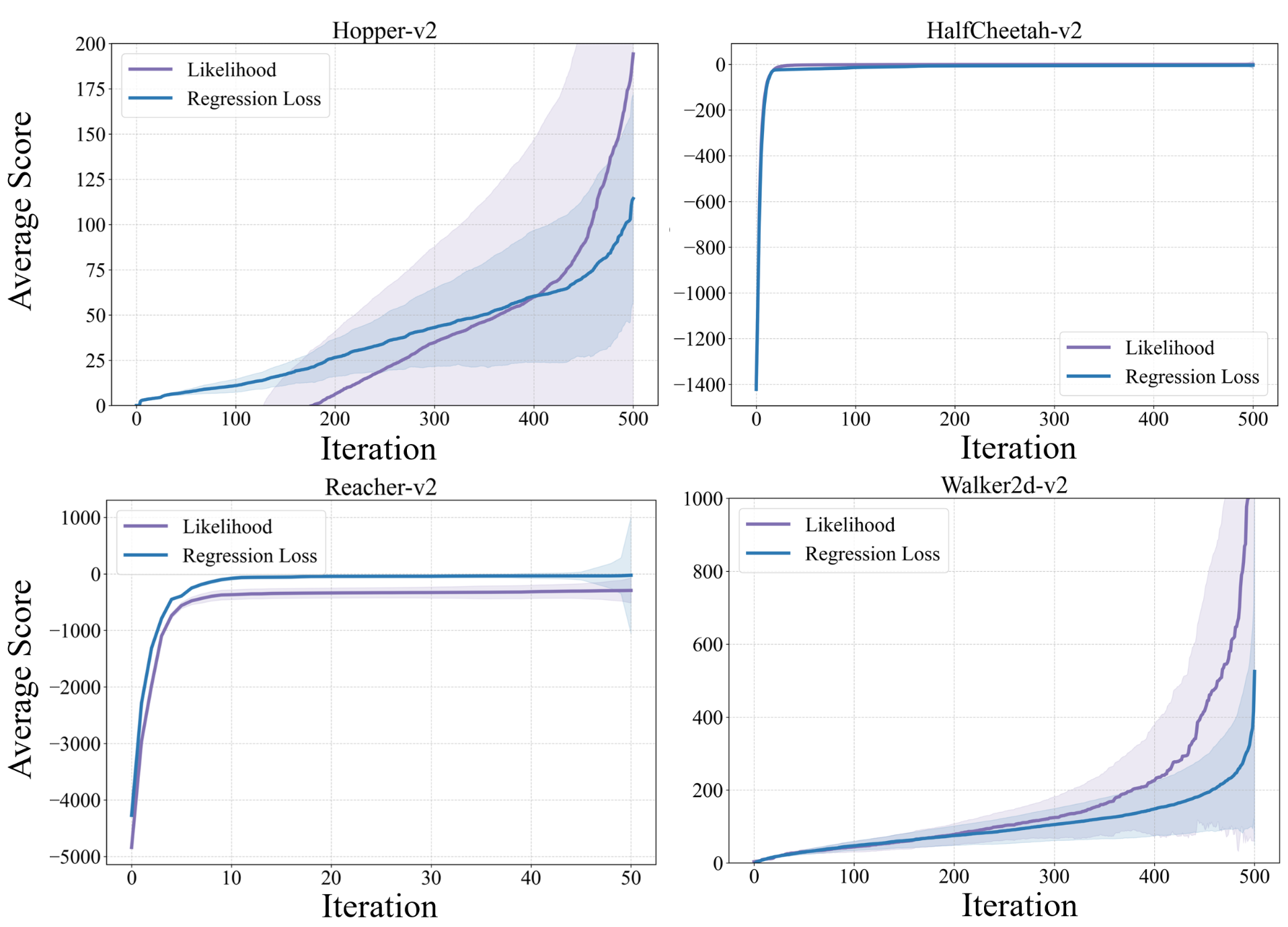}
}
    \caption{Performance of causal exploration with different forms of intrinsic rewards.}
    \label{fig:mujoco_change_model}
\end{figure*}

\end{document}